\documentclass[3p]{elsarticle}

\usepackage[utf8]{inputenc} 
\usepackage[T1]{fontenc}    
\usepackage{hyperref}       
\usepackage{url}            
\usepackage{booktabs}       
\usepackage{amsfonts}       
\usepackage{nicefrac}       
\usepackage{microtype}      
\usepackage[table]{xcolor} 
\usepackage{amsmath}
\usepackage{amsthm}
\usepackage{algorithm}
\usepackage{algorithmic}
\usepackage{bm}
\usepackage{multirow}
\usepackage{graphicx}
\usepackage{subcaption}
\usepackage{amssymb}

\theoremstyle{plain}
\newtheorem{theorem}{Theorem}[section]
\newtheorem{proposition}[theorem]{Proposition}
\newtheorem{lemma}[theorem]{Lemma}
\newtheorem{corollary}[theorem]{Corollary}
\theoremstyle{definition}

\theoremstyle{remark}
\newtheorem*{remark}{Remark}
\newtheorem{ngassumption}{Assumption}

\DeclareMathOperator*{\argmax}{argmax}
\newcommand{\iter}[1]{{(#1)}}
\newcommand{\res}[2]{%
  \begin{tabular}{@{}c@{}}#1 \\ \scriptsize{(#2)}\end{tabular}%
}

\begin{document}

\begin{frontmatter}

\title{Online Adaptive Kernel Mixing for Gaussian Process Decision Making}

\author[IIITH]{Kavin Aravindan}{\corref{cor1}}
\author[IIITH]{Mani Tej Sriram}
\author[ASU]{Gautam Dasarathy}
\author[IIITH]{Tejas Bodas}
\affiliation[IIITH]{op={},
  organization={International Institute of Information Technology, Hyderabad}}
\affiliation[ASU]{op={},
  organization={Arizona State University}}

\nonumnote{Code available at \url{https://github.com/kavin-aravindan/HACK-GPs}. Email: \texttt{kavin.aravindan@research.iiit.ac.in}}

\begin{abstract}
  Gaussian Processes (GPs) are widely used as surrogates for black-box functions in sequential decision-making problems such as Bayesian optimization (BO), level set estimation (LSE), and Bayesian active learning (BAL). GP performance critically depends on kernels, and standard kernels can lead to suboptimal decisions under misspecification. To address this, we introduce \textbf{HACK GPs} (Hedge Adaptive Cumulative Kernels), a method that views kernel selection as an online learning with expert advice problem. HACK treats each candidate kernel as a GP ``expert'' and updates a distribution over experts online using AdaHedge, based on a loss received as a proxy for their ability to fit the function and align with the task objective. We provide two variants of HACK: (i) Mixture of Gaussians (MoG) and (ii) categorical sampling. We establish general guarantees showing that, under a loss-gap condition, the weight concentrates on the best kernel and the resulting acquisition function is close to that of the best expert. Empirically, we observe robust performance across BO, LSE, and BAL compared to standard kernels such as Squared Exponential and Matérn-5/2, as well as simple ensemble baselines.
\end{abstract}


\end{frontmatter}

\section{Introduction}
Many experimental setups in science, engineering and machine learning require the optimization and modeling of noisy, expensive-to-evaluate, black-box functions. Gaussian processes (GPs) have proven to be an effective choice for these tasks due to their sample efficiency and ability to model uncertainties \cite {rasmussen2006gaussian}.

Several of these tasks are sequential in nature: at each round $t$, a surrogate model proposes an input $x_t$, observes a noisy evaluation $y_t$ of the black-box function $f$, and updates its posterior before the next decision. This template underlies Bayesian Optimization (BO), which seeks the maximizer $x^* =\argmax_{x\in\mathcal{X}}f(x)$ under a limited budget; Level Set Estimation (LSE), which aims to recover a superlevel set $S_h := \{x \in \mathcal{X} : f(x) \ge h\}$; and Bayesian Active Learning (BAL), where one queries to maximally reduce uncertainty about a target quantity. GPs are well suited to these settings because of their ability to convert the posterior model's uncertainty estimates into sampling strategies called acquisition functions. It is essential to ensure the quality of the GP in order to have reliable performance in these tasks.

GP performance is governed by the choice of the kernel, which encodes assumptions such as smoothness, stationarity, and the length scales across which such phenomena occur in the underlying function. {\em Kernel misspecification}, where these assumptions are misaligned with the underlying function, can lead to poorly calibrated uncertainty estimates and compounding errors in sequential decision-making. Despite this, standard kernels such as squared exponential and Matérn-5/2 are often used in practice, even when they may be ill-suited for the task at hand \cite {snoek2012,pmlr-v54-gardner17a}. 

To address this issue, we propose \textbf{HACK GPs} (\textbf{H}edge \textbf{A}daptive \textbf{C}umulative \textbf{K}ernels), a simple alternative that views kernel selection through the lens of online learning with expert advice.
We treat each candidate kernel as a GP ``expert'', maintain a distribution of weights over experts, and update the distribution online using AdaHedge \cite {adahedge}, a parameter-free variant of Hedge \cite {FREUND1997119} with a time-varying learning rate.
Experts are scored with per-round \emph{probabilistic} losses derived from their predictive performance, such as negative log-likelihood, optionally augmented with task-dependent proxies, allowing the weighting (and therefore the acquisition strategy) to adapt over time. 

We instantiate two use modes: (i) a \emph{mixture-of-Gaussians} (MoG) predictive distribution formed by weighting each expert's posterior, and (ii) \emph{categorical sampling} of a single kernel.
Both yield a plug-and-play procedure that progressively down-weights poorly aligned kernels, while boosting suitable ones, and incur only a modest computational overhead when the user-specified kernel candidate set is small. To place our approach in context, we next review existing methods for addressing kernel misspecification and model uncertainty in GP-based sequential decision making.

Most approaches that use GPs as surrogates make use of standard kernels not necessarily designed for the function at hand, such as the squared exponential or Mat\'ern-5/2 kernels \cite {snoek2012}. While the use of a general-purpose kernel is convenient, it can lead to suboptimal performance during tasks where the kernel does not align well with the nature of the function it is fitting \cite {experimentalAdaptive}, referred to as \emph{kernel misspecification}. It is also known that with overly general kernel choices, tasks like BO may converge slowly on complex functions in moderate to high dimensions \cite {pmlr-v54-gardner17a}. Early attempts to address kernel misspecification include compositional kernel structure search for GP regression \cite {duvenaud2013structurediscoverynonparametricregression}, where a grammar is used to generate new kernel compositions that may fit the data better. Parallel work has focused on flexible kernel families, such as spectral mixture kernels, which are able to approximate a broad class of stationary covariances \cite {wilson2013gaussianprocesskernelspattern}. However, these methods often require significant data before their selection criteria can separate candidates effectively. A modern variant addresses this by leveraging the generalizability of LLMs to suggest and combine kernels, achieving an effective, expressive family without the need for a large number of samples \cite {suwandi2025adaptivekerneldesignbayesian}.

Ensemble methods offer an alternative approach by maintaining multiple GP models and combining their predictions or decisions, rather than committing to a single kernel throughout the sequential process. Roman et al.~\cite{experimentalAdaptive} investigated multiple heuristics for GP ensembling. Formalizing this, Lu et al.~\cite{lu2022surrogatemodelingbayesianoptimization} applied Bayesian Model Averaging (BMA) over kernels, while Ravishankar et al.~\cite{ravishankar2024ensemble} extended the ensemble approach to both kernels and acquisition functions. Sandberg et al.~\cite{anonymous2025efficient} apply BMA to general GP priors, while concurrently employing an elimination strategy, first introduced by Ziomek et al.~\cite{ziomek2025timevaryinggaussianprocessbandits}. Early GP-based LSE methods use confidence-bound heuristics, with Bryan et al.~\cite{bryan2005active} introducing the straddle algorithm, which samples points around the threshold with high variance. Gotovos~\cite{gotovos2013active} formalized this, utilizing GP-UCB-style bounds \cite {Srinivas_2012} to provide convergence guarantees. Alternatively, Ravishankar et al.~\cite{ravishankarimprovement} formulated the problem via EI, targeting points that reduce the gap to the threshold. However, these methods rely on the GP being well-specified for the black-box function. To address potential mismatches, Zanette et al.~\cite{zanette2018robust} proposed an acquisition function robust to model misspecification, RMILE.

GP-based BAL for regression typically uses uncertainty-driven acquisition strategies, with predictive entropy being a standard choice \cite{mackay1992information,settles2009active}. An alternative choice is Bayesian Active Learning by Disagreement (BALD) introduced by Houlsby et al.~\cite{houlsby2011bayesian}, which proposes maximizing the expected decrease in predictive entropy. In an attempt to address hyperparameter misspecification in BAL, Riis et al.~\cite{riis2022bayesian} propose using a fully Bayesian treatment of the GP hyperparameters, providing other acquisition function variants that use the fully Bayesian nature of the GP. Polyzos et al.~\cite{EnsemblesBAL} formulated an adaptive ensemble GP method using BMA to handle kernel misspecification, utilizing a weighted mixture of acquisition functions.

Taken together, existing approaches address model misspecification through the use of more expressive kernel families, or through model averaging or selection, or through the use of task-specific robust acquisition strategies. HACK complements these directions by treating kernel adaptation itself as an online expert-advice problem, allowing the relative importance of candidate kernels to evolve during sequential decision making according to their performance and alignment with the task. The main contributions of this work are as follows.

\subsection*{Contributions:}
\begin{itemize}
    \item We introduce \textbf{HACK GPs}, a framework for online adaptive kernel weighting in Gaussian process–based sequential decision making, formulated as an expert-advice problem over a finite set of candidate~kernels.
    \item We develop two practical instantiations: (i) HACK-MoG, which combines expert posterior predictive distributions into a mixture-of-Gaussians predictive distribution, and (ii) HACK-Cat, which samples a single kernel per round, both of which integrate seamlessly with standard acquisition-driven methods.
    \item We establish theoretical guarantees showing that, under a loss-gap condition, the expert weights concentrate on the best kernel in hindsight and that this concentration translates to closeness of the acquisition function.
    \item We empirically demonstrate across BO, LSE, and BAL benchmarks that adaptive kernel weighting improves robustness to kernel misspecification compared to standard single-kernel and simple ensemble~methods.
\end{itemize}

\section{Background}
\label{sec:background}
\subsection{Gaussian Processes}
A Gaussian Process (GP) is a stochastic process such that any finite collection of its random variables has a multivariate Gaussian distribution.
GPs are often used as surrogate models to approximate the black-box function $f$ in BO, LSE and BAL. GPs are parametrized by a mean function $\mu$ and a covariance function given by a kernel $k(\cdot,\cdot)$.
We model the black-box function $f$ with the GP prior $f(x) \sim \mathcal{GP}\!\left(\mu(x),\, k(x,x')\right).$
The mean function $\mu(x)$ provides the mean at point $x$, while the kernel $k(x,x')$ determines the covariance between function values at two inputs $x$ and $x'$. We fit the GP on noisy observations $y_i = f(x_i) + \epsilon_i$ with $\epsilon_i \sim \mathcal{N}(0,\sigma_n^2)$.
Let $\mathcal{D}_t=\{(x_i,y_i)\mid i=1,2,\dots,t\}$, $X=[x_1,\dots,x_t]^\top$, and $\mathbf{y}=[y_1,\dots,y_t]^\top$.
Define the kernel matrix $K \in \mathbb{R}^{t\times t}$ by $K_{ij}=k(x_i,x_j)$, and the cross-covariance vectors
$k(x,X) = [k(x,x_1),\dots,k(x,x_t)]$ and $k(X,x)=k(x,X)^\top$.
Let $Y_{t+1}=f(x_{t+1})+\epsilon_{t+1}$ denote the noisy observation at a test point $x_{t+1}$. After conditioning the GP on $\mathcal{D}_t$, the posterior distributions of the function value and the noisy observation are
\begin{align*}
f(x_{t+1})\mid \mathcal{D}_t, x_{t+1}
&\sim \mathcal{N}\!\left(\mu_{f\mid \mathcal{D}_t}(x_{t+1}),\, \sigma_{f\mid \mathcal{D}_t}^2(x_{t+1})\right),\\
Y_{t+1} \mid \mathcal{D}_t, x_{t+1}&\sim\mathcal{N}\!\left(\mu_{f\mid \mathcal{D}_t}(x_{t+1}),\, \sigma_{f\mid \mathcal{D}_t}^2(x_{t+1})+\sigma_n^2\right), \\
\mu_{f\mid \mathcal{D}_t}(x) &= \mu(x) + k(x,X)\bigl(K+\sigma_n^2 I\bigr)^{-1}\bigl(\mathbf{y}-\mu(X)\bigr),\\
\sigma_{f\mid \mathcal{D}_t}^2(x) &= k(x,x) - k(x,X)\bigl(K+\sigma_n^2 I\bigr)^{-1}k(X,x).
\end{align*}
Refer to \cite{rasmussen2006gaussian} for more details. The kernel encodes assumptions about the nature of the function that is being fit (e.g., smoothness, stationarity). When the kernel is ill-suited for the underlying function, it can lead to subpar performance on tasks that use GPs, a phenomenon referred to as kernel misspecification.

Figure~\ref{fig:kernelmisspec} illustrates the effect of kernel misspecification on Gaussian Process regression using the 1D Ackley function. We fit GPs with a Squared Exponential (SE) kernel and a Periodic (PER) kernel to the same set of 15 uniformly sampled observations. While the SE kernel captures the smooth global structure of the function, it fails to model the underlying periodic components present in Ackley, leading to oversmoothing. In contrast, the Periodic kernel imposes strong periodic inductive bias, which results in spurious oscillations and poor extrapolation outside the observed regions. This example highlights how inappropriate kernel choices can significantly distort posterior mean and uncertainty estimates, motivating adaptive kernel selection.
In this work, we consider collections of GP models corresponding to different kernels, each inducing its own posterior and acquisition function.

\subsection{Bayesian Optimization}
Bayesian optimization (BO) is a sequential method for optimizing expensive-to-evaluate black-box functions $f:\mathcal{X}\to\mathbb{R}$, with the goal of finding $x^\ast \in \arg\max_{x\in\mathcal{X}} f(x)$.
An acquisition function $\alpha_t$ uses the posterior GP conditioned on $\mathcal{D}_t$ to select the next point to evaluate, via $x_{t+1}\in\arg\max_{x\in\mathcal{X}} \alpha_t(x)$.
A popular acquisition function is Expected Improvement (EI) \cite{EI}, given by $\mathrm{EI}(x) = \mathbb{E}[\max(f(x)-y^+, 0)]$, where $y^+$ is the best observed output till the current iteration and the expectation is over the distribution provided by the GP posterior at $x$. 
Other common acquisition functions include UCB and KG~\cite {Srinivas_2012, frazier2009knowledge}.

\begin{figure}[H]
    \centering
    \includegraphics[width=1\linewidth]{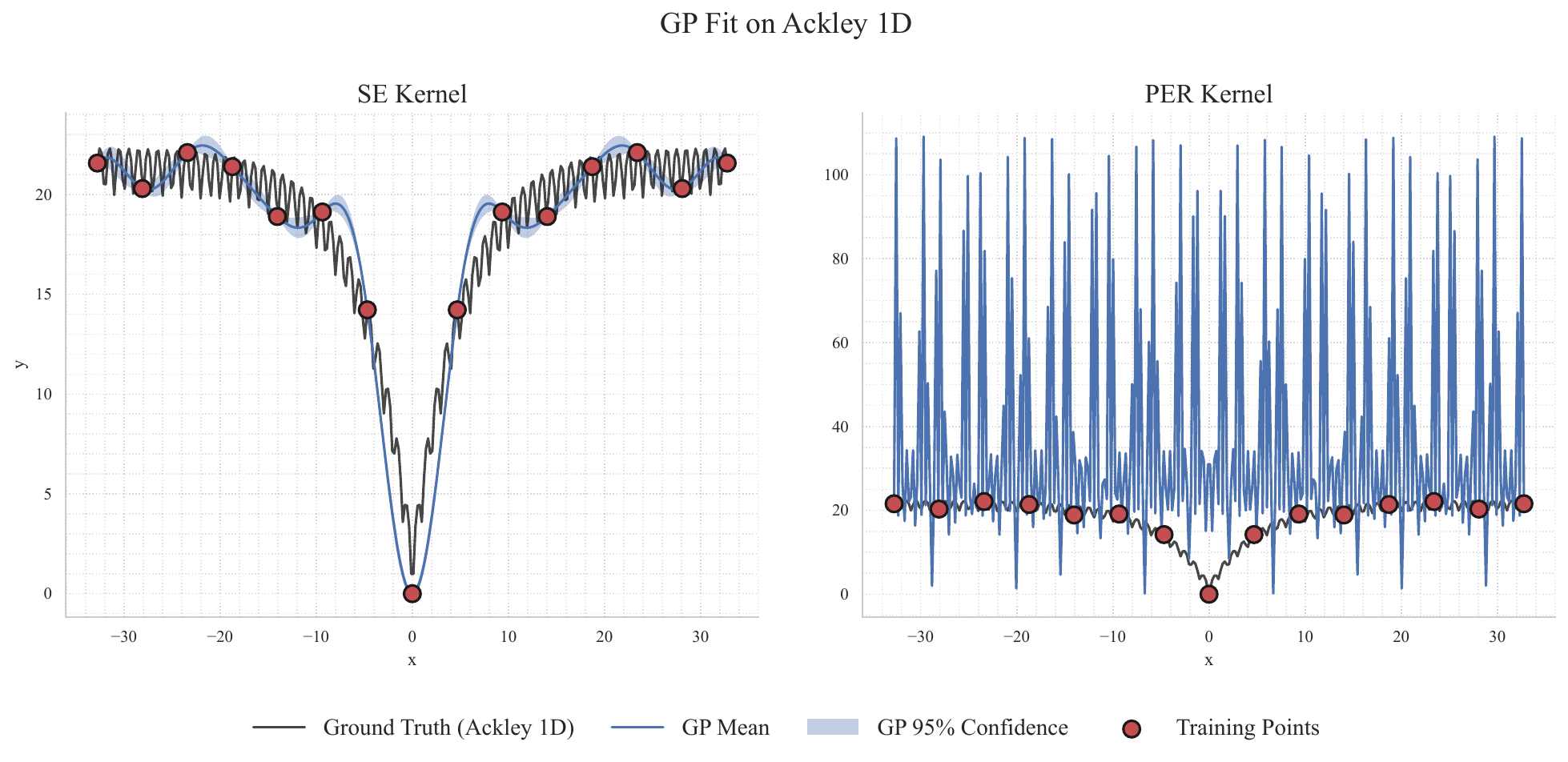}
    \caption{GP Fit on Ackley 1D function for 15 sample data points.}\label{fig:kernelmisspec}
\end{figure}

\subsection{Level Set Estimation}
Level set estimation (LSE) aims to identify the super-level set $S_h := \{x \in \mathcal{X} : f(x) \ge h\}$ of an expensive black-box function using a limited number of queries.
We focus on the setting where $\mathcal{X}$ is a fixed, finite set, following Gotovos~\cite{gotovos2013active}, and $h$ is provided as an explicit threshold.
Similar to BO, LSE proceeds sequentially by selecting the next evaluation point via an acquisition function, $x_{t+1} \in \arg\max_{x\in\mathcal{X}} \alpha_t(x)$. EI-LSE \cite {ravishankarimprovement} is an acquisition function based on EI for LSE and is given by \(\text{EI-LSE}(x)=\mathbb{E}[\max(g^*-|f(x)-h|,0)],\) where $g^*=\min_{1\leq i\leq t}|f(x_i)-h|$. Additionally, an exploration term $\beta \sigma^2(x)$ is often added, where $\beta$ is a hyperparameter. Other acquisition functions include straddle \cite {bryan2005active} and its generalization \cite {gotovos2013active}, TruVaR \cite {truvar2016} and RMILE \cite {zanette2018robust}.

\subsection{Bayesian Active Learning for Regression}
Bayesian active learning (BAL) involves querying labels for unlabelled data such that each query maximally reduces the uncertainty about the target function $f$ \cite{settles2009active}. We focus on BAL for regression tasks, operating on a finite unlabelled dataset $\mathcal{X}$, for which we seek to obtain continuous labels. Similar to BO and LSE, BAL selects points via $x_{t+1}\in\arg\max_{x\in\mathcal{X}}\alpha_t(x)$. The most common acquisition functions are predictive entropy \cite{mackay1992information},
which is a monotonic function of the posterior variance for a Gaussian, so maximizing it
is equivalent to using $\alpha_t(x)=\sigma_t^2(x)$, and Bayesian Active Learning by
Disagreement (BALD), which maximizes the expected decrease in predictive entropy
\cite{houlsby2011bayesian}.

\subsection{AdaHedge Algorithm}
\label{sec:adahedge}
We use AdaHedge \cite{adahedge}, a parameter-free variant of Hedge \cite{FREUND1997119} for prediction with expert advice.
At each round $t$, the learner maintains a distribution $w^{(t)}$ over $M$ experts and then
observes a loss vector $\ell^{(t)} \in [0,1]^M$.
Define the learner's \emph{expected (linear) loss} at round $t$ as $h^{(t)} \;:=\; \sum_{m=1}^M w_m^{(t)} \,\ell_m^{(t)}.$ Given a learning rate $\eta^{(t)}$, the corresponding \emph{mix loss} is defined as:
\[m^{(t)} \;:=\; -\frac{1}{\eta^{(t)}} \log\left(\sum_{m=1}^M w_m^{(t)} \exp\!\big(-\eta^{(t)}\ell_m^{(t)}\big)\right).\]
Furthermore, the \emph{per-round mixability gap} is then defined as
$\delta^{(t)} \;:=\; h^{(t)} - m^{(t)} \;\ge\; 0$
and the \emph{cumulative mixability gap} is defined as $\Delta_{mix}^{(t)} := \sum_{s=1}^t \delta^{(s)}$
with $\Delta_{mix}^{(0)}=0$. Letting $L_m^{(t-1)}=\sum_{s=1}^{t-1}\ell_m^{(s)}$, AdaHedge chooses a data-dependent learning rate 
\begin{equation}
    \eta^{(t+1)} \;=\; \ln M/\Delta_{mix}^{(t)} ,\label{eq:adahedgelr}
\end{equation}
(with the convention $\eta^{(1)}=\infty$ when $\Delta^{(0)}=0$, i.e.\ Follow-the-Leader initially), and sets
\[
w_m^{(t+1)} =
\frac{w_m^{(1)} \exp\!\big(-\eta^{(t+1)} L_m^{(t)}\big)}{\sum_{i=1}^Mw_i^\iter{1}\exp\big(-\eta^\iter{t+1}L_i^\iter{t}\big)}.
\]
After observing $\ell^{(t)}$, it updates $\Delta_{mix}^{(t)}=\Delta_{mix}^{(t-1)}+\delta^{(t)}$ using the definition above.
AdaHedge retains the worst-case $O(\sqrt{T\log M})$ regret guarantee while adapting to easier sequences via
this self-tuned learning rate. Refer to \cite{adahedge} for more details.

\section{Methodology}
\label{sec:methodology}
Given a finite set of kernels, our goal is to adaptively emphasize kernels that are most useful for a
sequential GP decision task. We cast this as online learning with expert advice: each kernel defines
a GP expert, and we update a distribution over experts based on their per-round predictive performance
and task-relevant progress.

Formally, we consider a sequential decision process over rounds $t=1,\dots,T$.
At round $t$, the learner selects a query $x_t \in \mathcal{X}$, observes $y_t = f(x_t) + \epsilon_t$,
and appends $(x_t,y_t)$ to the dataset $\mathcal{D}_{t-1}$. 
The goal is to choose the sequence $\{x_t\}_{t=1}^T$ so as to optimize a task-specific objective (e.g., BO, LSE, or BAL), which is operationalized
through a task-specific acquisition function.

Given a set of $M$ kernels \(\mathcal K=\{k_1,\ldots,k_M\}\), we associate each kernel \(k_m\) with GP expert \(m\). At the start of round \(t\), expert \(m\) is conditioned on \(\mathcal{D}_{t-1}\), inducing the posterior predictive distribution $p_{t,m}(\cdot\mid x) \sim \mathcal N\!\left(\mu_{t,m}(x),\sigma_{t,m}^2(x)\right)$, where \(\mu_{t,m}(x)\) and \(\sigma_{t,m}^2(x)\) denote the posterior predictive mean and variance of expert \(m\) at \(x\), respectively, with observation noise included when applicable. Let \(Y_{t,m}\) denote the random variable associated with expert \(m\), so that $Y_{t,m} \mid x \sim p_{t,m}(\cdot\mid x)$. Hereafter, \(m\) indexes both the expert and its associated kernel, and dependence on \(\mathcal{D}_{t-1}\) is suppressed since it is implicit in the round index~\(t\).
We maintain a vector of expert weights $w^{(t)}$ with initialization $w_m^{(1)} = 1/M$ and
$\sum_{m=1}^M w_m^{(t)} = 1$.
Different experts can prefer different query locations at the same round $t$ due to differing inductive biases encoded by their kernels. Our algorithm uses the current weight distribution either to combine expert posteriors into a single acquisition function or to sample a single expert, then uses the observed $(x_t,y_t)$ to assign per-expert losses and update
$w^{(t)}$ via AdaHedge. Task-specific acquisition constructions and loss definitions are given in Sections~5--7. A brief outline of the proposed algorithm is as~follows:
\begin{enumerate}
    \item Select a point $x_t$ (Algorithms \ref{alg:mog-point-selection} or \ref{alg:categorical-point-selection}) and observe $y_t$ from the black-box $f$.
    \item Compute each expert's loss $\ell_m^{(t)}$ using its posterior predictive distribution at the queried pair $(x_t,y_t)$.
    \item Update the expert weights via AdaHedge and refit each GP on $\mathcal{D}_t = \mathcal{D}_{t-1} \cup \{(x_t,y_t)\}$.
\end{enumerate}

\subsection{AdaHedge for Kernel Mixing} \label{methodology:online}

After querying data point $x_t$ at iteration $t$, each kernel/expert $m$ incurs a \textit{loss} denoted by $\ell_m^{(t)}$. We assume that $\ell_m^\iter{t}\in[0,1]$; smaller values of $\ell$ are preferred. $L_m^\iter{t} = \sum_{i=1}^t\ell_m^\iter{i}$ denotes the cumulative loss of expert $m$ at iteration $t$. The weights for each expert $m=1,\dots,M$ are updated with an exponential (multiplicative-weights) rule using the AdaHedge \emph{learning-rate rule} \cite{adahedge} and are given by
\[
w_m^{(t+1)} =\frac{w_m^{(1)} \exp\!\big(-\eta^{(t+1)} L_m^{(t)}\big)}{\sum_{i=1}^Mw_i^\iter{1}\exp\big(-\eta^\iter{t+1}L_i^\iter{t}\big)}. 
\] Kernels that repeatedly incur smaller losses accumulate exponentially larger weight over time, while kernels with larger losses are down-weighted. AdaHedge provides a conservative online weighting rule whose effective learning rate adapts to the observed separability of the experts. When losses strongly distinguish kernels, the learning rate permits rapid concentration. When losses are similar or noisy, the update remains less aggressive, which is desirable because premature kernel commitment can degrade early round performance.

We consider two types of loss functions. The first is negative log likelihood (NLL) that is designed towards better fitting of the black-box function. The second type of loss function is task specific, and its discussion is deferred to later. 
Note that NLL measures the ``surprise'' of a model at seeing the observation $(x_t, y_t)$ and is equivalent to maximizing the likelihood of the data given the model. For likelihood-based losses, we define the raw loss $\tilde{\ell}_m^{(t)} := -\log p_{t,m}(y_t \mid x_t)$. To ensure bounded losses as required by AdaHedge, we apply an affine normalization and clipping: 
\begin{equation}
\ell_m^{(t)} :=
\min(1,\;(\tilde{\ell}_m^{(t)} - \min_{j} \tilde{\ell}_j^{(t)})/C),\label{eq:nll_loss}
\end{equation}
where $C>0$ is a fixed scaling constant and is a hyperparameter. This transformation preserves relative loss differences while ensuring $\ell_m^{(t)} \in [0,1]$. AdaHedge is invariant to additive shifts of the losses. We additionally note that the scaling and clipping are intended to ensure weight stability and avoid abrupt changes observed empirically~without this transformation.

\subsection{Query Selection}
We propose two different methods of using the expert weight distribution to decide the next query point: \emph{mixture of Gaussians (MoG)} and \emph{categorical sampling}. 

\subsubsection{Mixture of Gaussians (MoG)}
\label{sec:mog}
At the start of round $t$, each expert $m$ provides a posterior predictive
distribution $p_{t,m}(\cdot\mid x)$.
To combine experts at the predictive level, we introduce a kernel index variable
$M_t\sim\mathrm{Cat}(w^{(t)})$. Conditional on $M_t=m$, the distribution of the observation is
\(p_{t,m}(\cdot\mid x) \sim \mathcal N\!\big(\mu_{t,m}(x),\,\sigma^2_{t,m}(x)\big).\)
Marginalizing over $M_t$ yields the MoG predictive distribution:
\begin{equation*}
p_{t,\mathrm{MoG}}(y\mid x)
\;=\;
\sum_{m=1}^M w^{(t)}_m\,p_{t,m}(y\mid x).
\label{eq:mog-predictive}
\end{equation*}
We then construct the acquisition $\alpha_{t,\mathrm{MoG}}(x)$ using the MoG predictive distribution $p_{t,\mathrm{MoG}}(\cdot | x)$, and select the next query as $x_t \in \arg\max_{x\in\mathcal{X}} \alpha_{t,\mathrm{MoG}}(x)$.

\begin{algorithm}[h]
\caption{Query Selection via MoG}
\label{alg:mog-point-selection}
\begin{algorithmic}[1]
    \STATE {\bfseries Input:} search space $\mathcal{X}$; task $\Gamma$;
    GP posteriors $\{(\mu_{t,m}(\cdot),\sigma_{t,m}^2(\cdot))\}_{m=1}^M$;
    weights $\{w_m^{(t)}\}_{m=1}^M$.
    \STATE {\bfseries Output:} query point $x_t \in \mathcal{X}$.

    \STATE {\bfseries MoG predictive distribution:} for each $x \in \mathcal{X}$,
    \[
      p_{t, \mathrm{MoG}}(y\mid x) \;\gets\; \sum_{m=1}^M w_m^{(t)}\, p_{t,m}(y\mid x).
    \]
    \STATE {\bfseries Acquisition:} define $\alpha_{t,\mathrm{MoG}}(x)$ from $p_{t,\mathrm{MoG}}(\cdot\mid x)$ according to task $\Gamma$.
    \STATE {\bfseries Point selection:} $x_t \gets \arg\max_{x \in \mathcal{X}} \alpha_{t,\mathrm{MoG}}(x)$.
    \STATE {\bfseries return} $x_t$.
\end{algorithmic}
\end{algorithm}

\subsubsection{Categorical Sampling}
As an alternative to MoG, we can sample a single expert according to the current AdaHedge weights and obtain $x_t$ using that expert's posterior predictive distribution only.  We draw an expert index
$m_t$ from $M_t$
and then use the sampled expert's GP posterior predictive distribution
$p_{t,m_t}(\cdot\mid x)$
and its acquisition function $\alpha_{t,m_t}(x)$. We select the next query point by $x_t \in \arg\max_{x\in\mathcal X}\alpha_{t,m_t}(x).$ 
\begin{algorithm}[H]
\caption{Query Selection via Categorical Sampling}
\label{alg:categorical-point-selection}
\begin{algorithmic}[1]
    \STATE {\bfseries Input:} search space $\mathcal{X}$; task $\Gamma$;
    GP posteriors $\{(\mu_{t,m}(\cdot),\sigma_{t,m}^2(\cdot))\}_{m=1}^M$;
    weights $\{w_m^{(t)}\}_{m=1}^M$.
    \STATE {\bfseries Output:} query point $x_t \in \mathcal{X}$.

    \STATE For each $m=1,\dots,M$, define acquisition $\alpha_{t,m}(x)$ from $p_{t,m}(\cdot\mid x)$ according to task $\Gamma$.
    \STATE Sample kernel index $m_t$ from $M_t$.
    \STATE Set $x_t \gets \arg\max_{x \in \mathcal{X}} \alpha_{t,m_t}(x)$.
    \STATE {\bfseries return} $x_t$.
\end{algorithmic}
\end{algorithm}

\begin{algorithm}[h]
\caption{HACK Algorithm (task-agnostic)}
\label{alg:adaptive-kernel-hedge}
\begin{algorithmic}[1]
    \STATE {\bfseries Input:} search space $\mathcal{X}$; kernels $\mathcal{K}=\{k_1,\dots,k_M\}$;
    task $\Gamma$; point-selection routine \textsc{SelectQuery} (Alg.~\ref{alg:mog-point-selection} or \ref{alg:categorical-point-selection});
    task loss $\ell_{\Gamma}$; iterations $T$.
    \STATE {\bfseries Output:} task-specific estimate $\hat{\theta}_T$.

    \STATE Initialize data $\mathcal{D}_0$; fit GP posteriors $\{\mathcal{GP}^{(0)}_m\}_{m=1}^M$ on $\mathcal{D}_0$.
    \STATE Initialize weights $w_m^{(1)} \gets 1/M$ and cumulative losses $L_m^{(0)}\gets 0$ for all $m=1,\dots,M$.
    \FOR{$t = 1$ {\bfseries to} $T$}    
    \STATE \textbf{Point selection:} $x_t \gets \textsc{SelectQuery}(\mathcal{X},\{\mathcal{GP}_m^{(t-1)}\}_{m=1}^M,\bm{w}^{(t)},\Gamma)$.

        \STATE {\bfseries Query:} observe $y_t  =f(x_t)+\epsilon_t$.
        \STATE {\bfseries Loss computation:} for each $m=1,\dots,M$,
        \[\ell_m^{(t)} \gets \ell_{\Gamma}(\mathcal{GP}_m^{(t-1)},x_t,y_t,\mathcal{D}_{t-1}),\]\[L_m^{(t)} \gets L_m^{(t-1)}+\ell_m^{(t)}.\]
        \STATE {\bfseries Learning-rate update:} update $\eta^{(t+1)}$ according to the AdaHedge rule described in Equation \ref{eq:adahedgelr}.
        \STATE {\bfseries Weight update:} for all $m=1,\dots,M$,
               \[
                 w_m^{(t+1)} \gets 
                 \frac{w_m^{(1)} \exp\!\big(-\eta^{(t+1)} L_m^{(t)}\big)}{\sum_{i=1}^Mw_i^\iter{1}\exp\big(-\eta^\iter{t+1}L_i^\iter{t}\big)}.
               \]
        \STATE Update $\mathcal{D}_t \gets \mathcal{D}_{t-1}\cup\{(x_t,y_t)\}$ and update each $\mathcal{GP}^{(t-1)}_m$ using $\mathcal{D}_t$ to obtain posterior $\mathcal{GP}^{(t)}_m$.
    \ENDFOR
    \STATE {\bfseries return} $\hat{\theta}_T \gets \textsc{Decide}_{\Gamma}(\mathcal{D}_T,\{\mathcal{GP}^{(T)}_m\}_{m=1}^M,\bm{w}^{(T+1)})$.
\end{algorithmic}
\end{algorithm}

A generic task-agnostic version of our algorithm is presented in Algorithm \ref{alg:adaptive-kernel-hedge}. We hereafter refer to the use of Algorithm \ref{alg:adaptive-kernel-hedge} with MoG (Algorithm \ref{alg:mog-point-selection}) as HACK-MoG and with categorical selection (Algorithm \ref{alg:categorical-point-selection}) as HACK-Cat. The task-specific estimates $\hat{\theta}_T$ returned from the last step of Algorithm \ref{alg:adaptive-kernel-hedge} $\textsc{Decide}_{\Gamma}$ are:
\begin{enumerate}
    \item BO: $\hat{\theta}_T=\hat{x}\in\argmax_{(x_i,y_i)\in \mathcal{D}_T} y_i$;
    \item LSE: $\hat{\theta}_T = \{x\in\mathcal{X}:\sum_{m=1}^Mw_m^\iter{T+1}\mu_{m,T}(x)\geq h\}$;
    \item  BAL : $\hat{\theta}_T = \mathcal{D}_T$.
\end{enumerate}
The learning rate $\eta^{(t+1)}$ is updated at each iteration according to the AdaHedge rule given in Equation~\ref{eq:adahedgelr}, after computing the round-$t$ losses and before updating the expert weights. The GP experts are then updated using the augmented dataset $\mathcal{D}_t$.
Additionally, AdaHedge can also be replaced by standard Hedge~\cite{FREUND1997119}, results of which can be found in \ref{app:HedgeVsAdahedge}.

\subsection{Computational Complexity}
\label{sec:complexity}
Table \ref{tab:complexity} provides a breakdown of the computational complexity of both HACK-MoG and HACK-Cat in comparison with a standard single kernel GP decision making process. Acquisition function optimization is performed using multi-start L-BFGS and therefore requires multiple acquisition function evaluations. The complexities reported in Table \ref{tab:complexity} correspond to the cost of a single such evaluation. For HACK-Cat, each acquisition function evaluation requires the posterior of only the sampled expert, whereas for HACK-MoG it requires the posterior from all experts. In addition, both HACK variants evaluate all $M$ expert posteriors at the queried point to compute their losses.

We note that the primary increase in computational overhead in HACK comes from the need to refit all the $M$ kernels present in the family versus updating only 1 kernel for a standard GP. HACK-MoG incurs an additional factor of $M$ during acquisition evaluation because its mixture acquisition depends on all expert posteriors. Thus, for a fixed number of candidate kernels, the additional cost of HACK scales linearly with the size of the kernel family. In applications where black-box evaluations dominate computational cost and $M$ is small, this overhead may remain modest. We also provide an empirical wall-clock comparison with simple baselines in \ref{app:runtime}.

\begin{table}[H]
\centering
\caption{Per-iteration computational complexity.
$n$ denotes the number of observations and $M$ the number of candidate
kernels. Refer to Section \ref{sec:complexity} for more details.}
\label{tab:complexity}
\begin{tabular}{lccc}
\toprule
\textbf{Component}
& \textbf{Single Kernel}
& \textbf{HACK-Cat}
& \textbf{HACK-MoG} \\
\midrule

GP refitting
& $O(n^3)$
& $O(Mn^3)$
& $O(Mn^3)$ \\

Acquisition evaluation
& $O(n^2)$
& $O(n^2)$
& $O(Mn^2)$ \\

Expert-loss evaluation
& $O(n^2)$
& $O(Mn^2)$
& $O(Mn^2)$ \\

AdaHedge update
& --
& $O(M)$
& $O(M)$ \\

\midrule
\textbf{Dominant total}
& $O(n^3)$
& $O(Mn^3)$
& $O(Mn^3)$ \\
\bottomrule
\end{tabular}
\end{table}

\section{Theoretical Analysis}
\label{sec:theory}
We now provide some general theoretical results that are used for all three tasks considered in this paper. These results are intentionally stated for abstract, task-agnostic functions and generic expert losses. Later we will specialize these to each task. Proofs and technical details may be found in \ref{app:A1}.

\begin{lemma}
\label{lemma:meanvar}
For the MoG predictive distribution $p_{t,\mathrm{MoG}}(y\mid x)
=\sum_{m=1}^M w_m^{(t)}\,p_{t,m}(y\mid x)$, the mean and variance are
$\bar{\mu}_t(x)=\sum_{m=1}^M w_m^{(t)}\,\mu_{t,m}(x)$,
\[
\bar{\sigma}^2_t(x)
=
\underbrace{\sum_{m=1}^M w_m^{(t)}\,\sigma^2_{t,m}(x)}_{\text{within-kernel}}
+
\underbrace{\sum_{m=1}^M w_m^{(t)}\big(\mu_{t,m}(x)-\bar{\mu}_t(x)\big)^2}_{\text{between-kernel (disagreement)}}.
\]
\end{lemma}
\begin{proof}
    Let $Y_t \mid x \sim p_{t,\mathrm{MoG}}(\cdot\mid x)$. The mean follows directly from the law of total expectation:
    \[\mathbb{E}[Y_t\mid x]
     = \mathbb{E}_{M_t}[\mu_{t,m}(x)]= \sum_{m=1}^M w_m^\iter{t}\,\mu_{t,m}(x).\]
Using the law of total variance, we have
\begin{align*}
    \mathrm{Var}(Y_t\mid x)
    &= \mathbb{E}_{M_t}\!\Big[\mathrm{Var}\!\big(Y_t\mid (M_t=m),x\big)\Big]
    + \mathrm{Var}_{M_t}\!\Big(\mathbb{E}\!\big[Y_t\mid (M_t=m),x\big]\Big) \\
    &= \mathbb{E}_{M_t}\!\big[\sigma^2_{t,m}(x)\big]
    + \mathrm{Var}_{M_t}\!\big(\mu_{t,m}(x)\big) \\
    &= \sum_{m=1}^M w_m^\iter{t}\,\sigma^2_{t,m}(x)
    + \sum_{m=1}^M w_m^\iter{t}\Big(\mu_{t,m}(x) - \sum_{j=1}^M w_j^\iter{t}\,\mu_{t,j}(x)\Big)^2.
\end{align*}
Defining the mixture mean $\bar{\mu}_t(x) := \sum_{m=1}^M w_m^\iter{t}\,\mu_{t,m}(x)$, the above can be written as
\begin{equation*}
    \mathrm{Var}(Y_t\mid x,\mathcal{D}_{t-1})
    = \sum_{m=1}^M w_m^\iter{t}\,\sigma^2_{t,m}(x)
    + \sum_{m=1}^M w_m^\iter{t}\big(\mu_{t,m}(x)-\bar{\mu}_t(x)\big)^2. \qedhere
\end{equation*}
\end{proof}
We denote the between-kernel disagreement term by
\(
\alpha_{t,\mathrm{disag}}(x)
:=
\sum_{m=1}^{M} w_m^{(t)}
\left(\mu_{t,m}(x)-\bar{\mu}_t(x)\right)^2.
\) The equality in Lemma \ref{lemma:meanvar} is a standard mixture identity. The first term is the average of the per-expert posterior variances. The second term captures disagreement among experts and thus quantifies how much mixture uncertainty is driven by between-expert differences in predictive means.

The following two lemmas, Lemma~\ref{lemma:kernelconc} and Lemma~\ref{lemma:kernelconcexpectation} provide lower bounds for the number of iterations $T$ for AdaHedge to concentrate on the best expert under the assumption of a loss gap. Here, $m^*$ denotes the best-performing kernel under the chosen HACK loss. Proofs for both the AdaHedge and Hedge variant can be found in \ref{app:A1}. \ref{app:C4} also provides empirical support for our loss gap~assumption.

\begin{lemma}
    \label{lemma:kernelconc} Assume the per-round loss difference between the best expert $m^*$ and every other expert $m\neq m^*$ is bounded from below, $\ell_m^{(t)} - \ell_{m^*}^{(t)} \geq \Delta>0$. Define $a_M := 1 + \frac{2}{3}\ln M$. Then $w_{m^*}^\iter{T} \geq 1-\varepsilon$ for
    \[
    T \geq 1 +\max \left\{\frac{4a_M^2}{\ln M}, \frac{1}{\Delta^2\ln M }\ln^2\left(\frac{(M-1)(1-\varepsilon)}{\varepsilon}\right) \right\}.
\]
\end{lemma}
Lemma~\ref{lemma:kernelconc} is a reasonably straightforward consequence of the analysis of AdaHedge under a uniform deterministic gap. The bound makes explicit the dependence on the number of experts $M$, the gap $\Delta$, and the desired confidence level $ 1 - \varepsilon$. 

\begin{lemma}
    \label{lemma:kernelconcexpectation}
    Let $(\mathcal F_t)_{t \ge 0}$ be a filtration such that the loss vector
$\ell^{(t)} \in [0,1]^M$ is $\mathcal F_t$-measurable. Fix an expert
$m^* \in \{1,\ldots,M\}$. Suppose there exists $\Delta>0$ such that,
for every $m \neq m^*$ and every $t=1,\ldots,T-1$,  
    $\mathbb{E}[\ell_m^{(t)} - \ell_{m^*}^{(t)} | \mathcal{F}_{t-1}] \geq \Delta$. Define $a_M := 1 + \frac{2}{3}\ln M$. Then $w_{m^*}^\iter{T} \geq 1-\varepsilon$ with probability $1-\delta_0$ for
\begin{equation*}
\begin{split}
T \ge 1 + \max\Bigg\{&
\frac{4a_M^2}{\ln M},\;
\frac{4}{\Delta^2\ln M}\ln^2\!\left(\frac{(M-1)(1-\varepsilon)}{\varepsilon}\right),
\frac{32}{\Delta^2}\ln\!\left(\frac{M-1}{\delta_0}\right)
\Bigg\}.
\end{split}
\end{equation*}
\end{lemma}
 Lemma \ref{lemma:kernelconcexpectation} relaxes the deterministic gap to an expected gap conditioned on the history. The result again follows from AdaHedge martingale-style arguments and standard concentration bound arguments; details are in \ref{app:A1}.
\begin{remark}
 A positive loss gap of $\Delta$ occurs when there is a single best kernel present in $\mathcal{K}$, leading to weight concentration on the best kernel. Due to the $1/\Delta^2$ in the lower bound, a larger $\Delta$ leads to faster weight concentration. Conversely, when $\Delta \approx 0$ because multiple kernels perform similarly to the best one, weight concentration to a single kernel may require an arbitrarily large number of iterations. This is intentional, since the goal of HACK is to be adaptive and handle the downside risk of using a bad kernel, providing robustness. We demonstrate examples for both of the above cases in Section \ref{sec:boexp} and provide some empirical evidence for the assumptions in the above two lemmas in \ref{app:C4}.
\end{remark}

The next theorem formalizes the intuitive statement that if the MoG acquisition function is uniformly close to the best expert’s acquisition function, then the point selected by maximizing the MoG acquisition achieves nearly the same acquisition value under the best expert.
\begin{proposition}
    
\label{thm:acq-argmax}
For an iteration $t$ and expert $m^*$, suppose there exists $\delta \ge 0$ such that
\(
\sup_{x\in\mathcal X}\big|\alpha_{t,MoG}(x)-\alpha_{t,m^*}(x)\big| \le \delta.
\)
Let
$
x_t \in \arg\max_{x\in\mathcal X} \alpha_{t,MoG}(x)
\ \text{and} \
x_t^* \in \arg\max_{x\in\mathcal X} \alpha_{t,m^*}(x).
$
 Then
$
\alpha_{t,m^*}(x_t) \ge \alpha_{t,m^*}(x_t^*)-~2\delta.
$
\end{proposition}
\begin{proof}
By our hypothesis,
\(
\sup_{x\in\mathcal X}\big|\alpha_{t,\mathrm{MoG}}(x)-\alpha_{t,m^*}(x)\big| \le \delta,
\)
so for every $x\in\mathcal X$,
\[
\alpha_{t,m^*}(x) \ge \alpha_{t,\mathrm{MoG}}(x) - \delta
\quad\text{and}\quad
\alpha_{t,\mathrm{MoG}}(x) \ge \alpha_{t,m^*}(x) - \delta.
\]
Since $x_t \in \arg\max_{x\in\mathcal X}\alpha_{t,\mathrm{MoG}}(x)$, we have
$\alpha_{t,\mathrm{MoG}}(x_t)\ge \alpha_{t,\mathrm{MoG}}(x_t^*)$.
Therefore,
\[
\alpha_{t,m^*}(x_t)
\ge \alpha_{t,\mathrm{MoG}}(x_t) - \delta
\ge \alpha_{t,\mathrm{MoG}}(x_t^*) - \delta
\ge \alpha_{t,m^*}(x_t^*) - 2\delta,
\]
which proves the claim. \qedhere

\end{proof}

Proposition \ref{thm:acq-argmax} is a simple stability-of-argmax statement. It requires a uniform approximation guarantee of the MoG acquisition function by the best expert’s acquisition function over all $\mathcal{X}$.

The results above identify a generic mechanism: when the expert losses induce a persistent gap $\Delta$, the AdaHedge algorithm concentrates the weights on the best expert, and the mixture-of-Gaussians predictive distribution and its associated acquisition function become uniformly close to those of that expert. What remains is to specify, for each task, loss functions that reliably reflect task-relevant progress and to verify that the corresponding acquisition functions satisfy the regularity conditions required by the generic results. In the following sections, we instantiate this framework for BO, LSE, and BAL, showing how appropriate task-specific losses and acquisition constructions allow the generic concentration and stability arguments to translate into concrete, task-relevant consequences.

\section{Bayesian Optimization}

We first specialize the generic expert-advice framework and theoretical results to Bayesian Optimization (BO). We define a task-specific loss that rewards predictive improvement, describe the resulting acquisition construction under the MoG predictive distribution, and show how expert weight concentration translates into uniform closeness of the BO acquisition function.

\subsection{Loss and Acquisition function} \label{bo:acq_loss}
In BO, a natural notion of per-round task progress is whether the newly queried point improves upon the best observed value so far. This motivates losses that reward accurate prediction of improvement events, rather than purely predictive fit.
To capture this, we propose a Brier score based loss \cite{VERIFICATIONOFFORECASTSEXPRESSEDINTERMSOFPROBABILITY}  where the binary event indicates whether the new observation improves upon the existing maximum.
Expert $m$ gives a posterior predictive distribution for the next observation $Y_{t,m}\mid x_t\sim\mathcal N(\mu_{t,m}(x_t),\sigma^2_{t,m}(x_t))$. Define the predicted
improvement probability
\[\pi^{\textsc{bo}}_{t,m}(x_t):=\mathbb{P}(Y_{t,m}\ge y^*_{t-1}\mid x_t)
=\Phi\!\big((\mu_{t,m}(x_t)-y^*_{t-1})/\sigma_{t,m}(x_t)\big),\] where $\Phi$ is the
standard normal CDF. We then define the Brier loss as
\begin{equation*}
\ell^{(t)}_{m,\textsc{brier}}
:=\Big(\mathbf{1}\{y_t\ge y^*_{t-1}\}-\pi^{\textsc{bo}}_{t,m}(x_t)\Big)^2. 
\end{equation*}
The Brier score is a proper scoring rule for probabilistic binary events and is naturally bounded in $[0,1]$, making it well suited for expert-advice algorithms. To combine the model fit loss (NLL, Equation \ref{eq:nll_loss}) and task loss, we use a convex combination of the two and consider 
$
\ell^{(t)}_m=\alpha\,\ell^{(t)}_{m,\textsc{brier}}+(1-\alpha)\,\ell^{(t)}_{m,\textsc{nll}}.
$
The NLL term encourages overall predictive calibration, while the task loss focuses the weighting on task-relevant predictive behavior. We find that combining the two stabilizes learning in early rounds, when task-specific signals alone may be noisy. Empirically, we find that $\alpha=0.5$ provides a good tradeoff. Results with other loss functions are in \ref{app:C1}.
In all our BO experiments, we use the EI acquisition function. Due to linearity of expectation, the EI of the MoG predictive distribution is:
\[
\mathrm{EI}_{t,\mathrm{MoG}}(x)
=
\sum_{m=1}^M w_m^\iter{t}\,\mathrm{EI}_{t,m}(x).
\]
This linearity allows EI to be computed exactly under the MoG predictive distribution without approximation, preserving the semantics of EI while enabling adaptive kernel weighting.

\subsection{Theoretical Analysis}
We now show how the generic concentration and stability results from Section \ref{sec:theory} translate into acquisition-level guarantees for BO when EI is used as the acquisition function. The key property is the linearity of EI under the MoG predictive distribution. The following lemma shows that expert weight concentration directly implies uniform closeness of the EI acquisition.
\begin{lemma} 
\label{lemma:ei-closeness}
    For each round $t$, assume $\exists B_{\mathrm{EI}} < \infty$ such that   $EI_{t,m}(x) \le B_{\mathrm{EI}}$ for every expert $m$ and every $x\in \mathcal X$. If $w_{m^{*}}^\iter{t} \geq1 - \varepsilon$, then
    \[
    \sup_{x\in \mathcal{X}}|\mathrm{EI}_{t,\mathrm{MoG}} (x)-\mathrm{EI}_{t,m^*}(x)| \leq \varepsilon B_{\mathrm{EI}}.
    \]
\end{lemma}
\begin{proof}
By the MoG construction,
\(
\mathrm{EI}_{t,\mathrm{MoG}}(x)=\sum_{m=1}^M w^{(t)}_m\,\mathrm{EI}_{t,m}(x).
\)
Hence, for any fixed $x\in\mathcal X$,
\begin{align*}
\mathrm{EI}_{t,\mathrm{MoG}}(x)-\mathrm{EI}_{t,m^*}(x)
&= \sum_{m=1}^M w^{(t)}_m\,\mathrm{EI}_{t,m}(x) - \mathrm{EI}_{t,m^*}(x) \\
&= \sum_{m\neq m^*} w^{(t)}_m\Big(\mathrm{EI}_{t,m}(x)-\mathrm{EI}_{t,m^*}(x)\Big).
\end{align*}
Taking absolute values and applying the triangle inequality gives us
\[
\big|\mathrm{EI}_{t,\mathrm{MoG}}(x)-\mathrm{EI}_{t,m^*}(x)\big|
\le \sum_{m\neq m^*} w^{(t)}_m\,\big|\mathrm{EI}_{t,m}(x)-\mathrm{EI}_{t,m^*}(x)\big|.
\]
Since $0\le \mathrm{EI}_{t,m}(x)\le B_{\mathrm{EI}}$ for all $m$, we have
$\big|\mathrm{EI}_{t,m}(x)-\mathrm{EI}_{t,m^*}(x)\big|\le B_{\mathrm{EI}}$, and therefore
\[
\big|\mathrm{EI}_{t,\mathrm{MoG}}(x)-\mathrm{EI}_{t,m^*}(x)\big|
\le \sum_{m\neq m^*} w^{(t)}_m\,B_{\mathrm{EI}}
= (1-w^{(t)}_{m^*})B_{\mathrm{EI}}
\le \varepsilon\,B_{\mathrm{EI}}.
\]
\[
    \sup_x|\mathrm{EI}_{t,MoG} (x)-\mathrm{EI}_{t,m^*}(x)| \leq \varepsilon B_{\mathrm{EI}}. \qedhere
\]
\end{proof} 

This boundedness holds under standard GP regularity conditions (compact $\mathcal{X}$, continuous bounded kernels, and a positive noise variance ensuring well-defined posterior mean/variance), since $\mathrm{EI}_{t,m}(x)$ is continuous on $\mathcal{X}$ and hence attains a finite maximum.
We next recall the results characterizing EI regret that will serve as the basis for subsequent results. We begin by restating in our notation,  a key result from Nguyen et al.~\cite[Theorem~4]{nguyen2017regret} that requires the following assumptions.

\begin{ngassumption}
\label{ass:nguyen}
\leavevmode
\begin{enumerate}
    \item[(i)] $f$ belongs to the RKHS $\mathcal H_k$ associated with the GP kernel, with bounded RKHS norm $\|f\|_k<\infty$.
    
    \item[(ii)] The kernel is normalized, with $k(x,x)=1$ for all $x\in\mathcal X$.

    \item[(iii)] The GP confidence event stated below holds true with probability at least
    $1-\delta_{\mathrm{Ng}}$
    \[
    |\mu_{t-1}(x)-f(x)|
    \le
    \sqrt{\beta_t}\,\sigma_{t-1}(x),
    \qquad
    \forall\, t,\ \forall\, x\in\mathcal X.
    \]
    
    \item[(iv)] The EI stopping condition holds on the rounds under consideration: the EI value at the selected point remains at least $\kappa>0$.
\end{enumerate}
\end{ngassumption}

\begin{theorem}[Consequence of Nguyen et al.~\cite{nguyen2017regret}]
\label{thm:ng-thm4}
Fix a kernel $k_{m^*}$ satisfying Assumption \ref{ass:nguyen}. 
Suppose EI is run with the best-observed incumbent $y_t^{\max}$.
Let $\gamma_T^{m^*}$ denote the maximum information gain of expert $m^*$, let $\sigma_n^2$ be the observation-noise variance, and define
\[
\beta_T^{m^*}
::=
2\|f\|_{k_{m^*}}^2
+
300\gamma_T^{m^*}
\log^3\!\left(\frac{T}{\delta_{\mathrm{Ng}}}\right)\quad \text{and} \quad
C_\kappa
=
\log\!\left(\frac{1}{2\pi\kappa^2}\right).
\]
Then, with probability at least $1-\delta_{\mathrm{Ng}}$, the cumulative regret of EI using expert $m^*$ satisfies
\[
R_T
\le
\sqrt{
\frac{2T\gamma_T^{m^*}}
{\log(1+\sigma_n^{-2})}
}
\left[
\sqrt{3(\beta_T^{m^*}+1+C_\kappa)}
+
\sqrt{\beta_T^{m^*}}
\right]
=
O\!\left(
\sqrt{T\beta_T^{m^*}\gamma_T^{m^*}}
\right).
\]
\end{theorem}

 Theorem \ref{thm:ng-thm4} is conditional on Assumption \ref{ass:nguyen} holding for the expert $m^*$ on the query sequence considered. 
For HACK-MoG, the selected point is not necessarily the exact maximizer of the EI acquisition of $m^*$. 
However, once the weight on $m^*$ is large, Lemma~\ref{lemma:ei-closeness} shows that the MoG EI acquisition is uniformly close to the EI acquisition of $m^*$, and Proposition~\ref{thm:acq-argmax} converts this acquisition-level closeness into an additive regret slack.
The following theorem formalizes this transfer, showing how concentration of the HACK weights converts the fixed-expert EI regret guarantee into a regret bound for HACK-MoG.
\begin{theorem} 
\label{thm:hack-ei-regret-transfer}
Fix a horizon $T$ and a burn-in time $T_0<T$. Let
$I_T:=\{T_0+1,\dots,T\}$ and $n_T:=T-T_0$. Fix an expert $m^*$ using EI satisfying Assumption \ref{ass:nguyen}. Let $x_t^H$ denote the HACK-MoG query at
round $t$ and define
\(
r_t^H:=f(x^*)-f(x_t^H)
\) and \(
s_T^H:=f(x^*)-\max_{1\le t\le T}f(x_t^H).
\)
Assume that the event
$\mathcal E_{\mathrm{conc}}
:=
\{w_{m^*}^{(t)}\ge 1-\epsilon_t \text{ for every } t\in I_T\}$
holds with probability at least $1-\delta_{\mathrm{conc}}$, and that
there exists $B_{\mathrm{EI}}<\infty$ such that
$\mathrm{EI}_{t,m}(x)\le B_{\mathrm{EI}}$ for all
$t\in I_T$, $m\in\{1,\dots,M\}$, and $x\in\mathcal X$.
Then, with probability at least $1-\delta_{\mathrm{Ng}}-\delta_{\mathrm{conc}}$,
\[
\sum_{t=T_0+1}^{T} r_t^H
\le
\mathcal R_{T,T_0}^{\mathrm{Ng}}(m^*)
+
2\sum_{t=T_0+1}^{T}\epsilon_tB_{\mathrm{EI}}
\quad \text{and} \quad
s_T^H
\le
\frac{\mathcal R_{T,T_0}^{\mathrm{Ng}}(m^*)}{n_T}
+
\frac{2}{n_T}
\sum_{t=T_0+1}^{T}\epsilon_tB_{\mathrm{EI}},
\]
where \(\mathcal R_{T,T_0}^{\mathrm{Ng}}(m^*)\) denotes the Nguyen et al. cumulative regret term applied to expert \(m^*\) over the post-burn-in horizon \(n_T\), and is given by
\[
\mathcal R_{T,T_0}^{\mathrm{Ng}}(m^*)
:=
\sqrt{
\frac{2n_T\gamma_T^{m^*}}
{\log(1+\sigma_n^{-2})}
}
\left[
\sqrt{3(\beta_T^{m^*}+1+C_\kappa)}
+
\sqrt{\beta_T^{m^*}}
\right].
\]
\end{theorem}
\begin{proof}
    We work on the event $\mathcal E_{\mathrm{conc}}$ and on the GP confidence event
from assumption \ref{ass:nguyen}; by a union bound, their
intersection has probability at least
$1-\delta_{\mathrm{conc}}-\delta_{\mathrm{Ng}}$. We analyze only the post-burn-in
interval $I_T=\{T_0+1,\dots,T\}$. This is sufficient for simple regret because
\[
s_T^H
=
f(x^*)-\max_{1\le t\le T}f(x_t^H)
\le
f(x^*)-\max_{t\in I_T}f(x_t^H).
\]
Thus, ignoring the first $T_0$ rounds can only make the upper bound more
conservative. Fix $t\in I_T$. Since $w_{m^*}^{(t)}\ge 1-\epsilon_t$, Lemma~\ref{lemma:ei-closeness}
implies
\(
\sup_{x\in\mathcal X}
\left|
\mathrm{EI}_{t,\mathrm{MoG}}(x)
-
\mathrm{EI}_{t,m^*}(x)
\right|
\le
\epsilon_tB_{\mathrm{EI}}.
\)
Since $x_t^H\in\arg\max_{x\in\mathcal X}\mathrm{EI}_{t,\mathrm{MoG}}(x)$,
Proposition~\ref{thm:acq-argmax} gives
\(
\mathrm{EI}_{t,m^*}(x_t^H)
\ge
\max_{x\in\mathcal X}\mathrm{EI}_{t,m^*}(x)
-
2\epsilon_tB_{\mathrm{EI}}.
\)
Therefore, in particular,
\begin{equation}
\mathrm{EI}_{t,m^*}(x^*)
\le
\mathrm{EI}_{t,m^*}(x_t^H)
+
2\epsilon_tB_{\mathrm{EI}}.
\label{eq:hack-ei-approx}
\end{equation}

\noindent On the GP confidence event in Assumption~\ref{ass:nguyen},
Nguyen et al.~\cite[Lemma~6]{nguyen2017regret} implies
\[I_t(x^*)\le\mathrm{EI}_{t,m^*}(x^*)+\sqrt{\beta_t^{m^*}}\sigma_{t-1,m^*}(x^*).\]
Substituting Equation \ref{eq:hack-ei-approx} into the previous confidence bound yields
\[
I_t(x^*)
\le
\mathrm{EI}_{t,m^*}(x_t^H)
+
\sqrt{\beta_t^{m^*}}\sigma_{t-1,m^*}(x^*)
+
2\epsilon_tB_{\mathrm{EI}}.
\]
Applying the remainder of the EI regret argument of
Nguyen et al.~\cite[Theorem~4]{nguyen2017regret} on the post-burn-in rounds
$I_T$, using the stopping condition in Assumption~\ref{ass:nguyen}, yields
\[
\sum_{t=T_0+1}^{T} r_t^H
\le
\mathcal R_{T,T_0}^{\mathrm{Ng}}(m^*)
+
2B_{\mathrm{EI}}
\sum_{t=T_0+1}^{T}\epsilon_t.
\]
Finally,
\[
s_T^H
\le
f(x^*)-\max_{t\in I_T}f(x_t^H)
=
\min_{t\in I_T} r_t^H
\le
\frac{1}{n_T}\sum_{t=T_0+1}^{T}r_t^H.
\]
Substituting the cumulative-regret bound gives
\[
s_T^H
\le
\frac{\mathcal R_{T,T_0}^{\mathrm{Ng}}(m^*)}{n_T}
+
\frac{2B_{\mathrm{EI}}}{n_T}
\sum_{t=T_0+1}^{T}\epsilon_t.
\qedhere
\]
\end{proof}

Theorem \ref{thm:hack-ei-regret-transfer} cleanly transfers the regret statement of Nguyen et al. \cite{nguyen2017regret} with the inclusion of an additional error term. The second term is the cost of not using \(m^*\) directly quantified using Lemma \ref{lemma:ei-closeness} and Proposition \ref{thm:acq-argmax}. Thus, the adaptive mixture inherits the fixed-expert EI guarantee whenever the cumulative residual mass \(\sum_{t>T_0}\epsilon_t\) is small. We now evaluate the order of this residual term under the AdaHedge dynamics and the loss gap assumptions from Lemma \ref{lemma:kernelconcexpectation}. It shows that, once the burn-in is large enough for the gap to be detected with high probability, the post-burn-in excess simple-regret term decays as \(O(1/n_T)\).

\begin{corollary}
\label{cor:error-rate}
Assume the conditional loss-gap condition of Lemma \ref{lemma:kernelconcexpectation} holds for \(m^*\) with gap \(\Delta > 0\).
Fix a horizon \(T\), let \(\delta_0 = \delta_{\mathrm{conc}}/n_T\), and define
\[
\epsilon_t
=
\frac{M-1}{M-1+\exp\!\left(\frac{\Delta}{2}\sqrt{(t-1)\ln M}\right)},
\]
for $t\in I_T$ and let $T_0$ satisfy
\[
T_0
\ge
\max\left\{
\frac{4a_M^2}{\ln M},\;
\frac{32}{\Delta^2}\ln\!\left(\frac{M-1}{\delta_0}\right)
\right\}.
\]
Then, for this fixed horizon, with $c=\tfrac{\Delta}{2}\sqrt{\ln M}$, the summation of $\epsilon_t$ post-burn-in is bounded by
\[
\sum_{t=T_0+1}^{T}\epsilon_t
\;\le\;
(M-1)\!\left[
1 + \frac{2}{c^2}\!\left(1+c\sqrt{T_0}\right)
\right]e^{-c\sqrt{T_0}}.
\]
Consequently, in
Theorem~\ref{thm:hack-ei-regret-transfer} the additional simple-regret term satisfies
\[
\frac{2B_{\mathrm{EI}}}{n_T}
\sum_{t=T_0+1}^{T}\epsilon_t
=
O\!\left(\frac{1}{n_T}\right).
\]
\end{corollary}

Corollary~\ref{cor:error-rate} shows that, under a persistent loss gap, the HACK-specific error term in Theorem~\ref{thm:hack-ei-regret-transfer} contributes only $O(1/n_T)$ to the post-burn-in simple regret. Therefore,
\[
s_T^H
\le
\frac{\mathcal R_{T,T_0}^{\mathrm{Ng}}(m^*)}{n_T}
+
O\!\left(\frac{1}{n_T}\right).
\]
Thus, once the weights concentrate on a clearly favored expert, the dominant term is the EI regret term inherited from Nguyen et al.~\cite{nguyen2017regret}, averaged over the post-burn-in horizon. The size of this inherited term is governed by the maximum information gain $\gamma_T^{m^*}$. For the squared exponential kernel in $d$ dimensions, Nguyen et al.~\cite{nguyen2017regret} use $\gamma_T = O((\log T)^{d+1})$ and obtain the cumulative regret order
$
R_T
=
O\!\left(
\sqrt{T(\log T)^{d+4}}
\right).
$
Consequently, when the $m^*$ kernel is SE, the corresponding post-burn-in contribution in Theorem~\ref{thm:hack-ei-regret-transfer} is
\[
\frac{\mathcal R_{T,T_0}^{\mathrm{Ng}}(m^*)}{n_T}
=
O\!\left(
\sqrt{\frac{(\log n_T)^{d+4}}{n_T}}
\right),
\]
replacing $T$ by the post-burn-in horizon $n_T$.

Overall, the BO guarantee is most informative in the separated-expert regime: one kernel has a persistent loss advantage, AdaHedge concentrates on it, and the remaining HACK-specific term vanishes as $O(1/n_T)$.

\subsection{Experiments}

\label{sec:boexp}
We evaluate on (i) synthetic functions from the Virtual Library of Simulation Experiments~\cite {simulationlib} and (ii) the 3D RobotPushing simulator benchmark~\cite {airbo}. Log simple regret, $\ln\!\big(f(x^*)-\max_{t\le T} y_t\big)$, is used as the performance metric, capturing how close the current observations are to the function's global maxima. Across all tasks, we run BO for $T=100$ iterations with EI, initializing with 5 random points, and report final log simple regret averaged over 20 seeds. All ensemble/compositional methods use the same kernel set \[\mathcal{K}=\{\textsc{SE},\textsc{Matern-}5/2,\textsc{Matern-}3/2,\textsc{RQ},\textsc{LIN},\textsc{PER}\}\] and we compare against common choices for fixed-kernel GP-BO (SE, Matérn-$5/2$), adaptive kernel-selection heuristics~\cite {experimentalAdaptive}, EGP~\cite {lu2022surrogatemodelingbayesianoptimization}, and CAKE~\cite {suwandi2025adaptivekerneldesignbayesian}.  \textsc{HACK-MoG} performs best overall, achieving the lowest average rank and most wins on the synthetic suite and the best performance on RobotPushing (details in ~\ref{app:robotpushing}); see Table~\ref{tab:bo_synthetic} and Figure~\ref{fig:BO_performance}. Note that the wider confidence intervals in the higher-dimensional BO problems are expected, since the larger search spaces permit greater variation in the optimization trajectories across runs.
We reiterate that the intended goal of HACK is not to provide a guaranteed improvement in performance but rather to prevent the undesirable case of using a misspecified kernel for BO, or any other sequential decision-making task.

\begin{figure*}[h]
    \centering
    \includegraphics[width=1.0\linewidth]{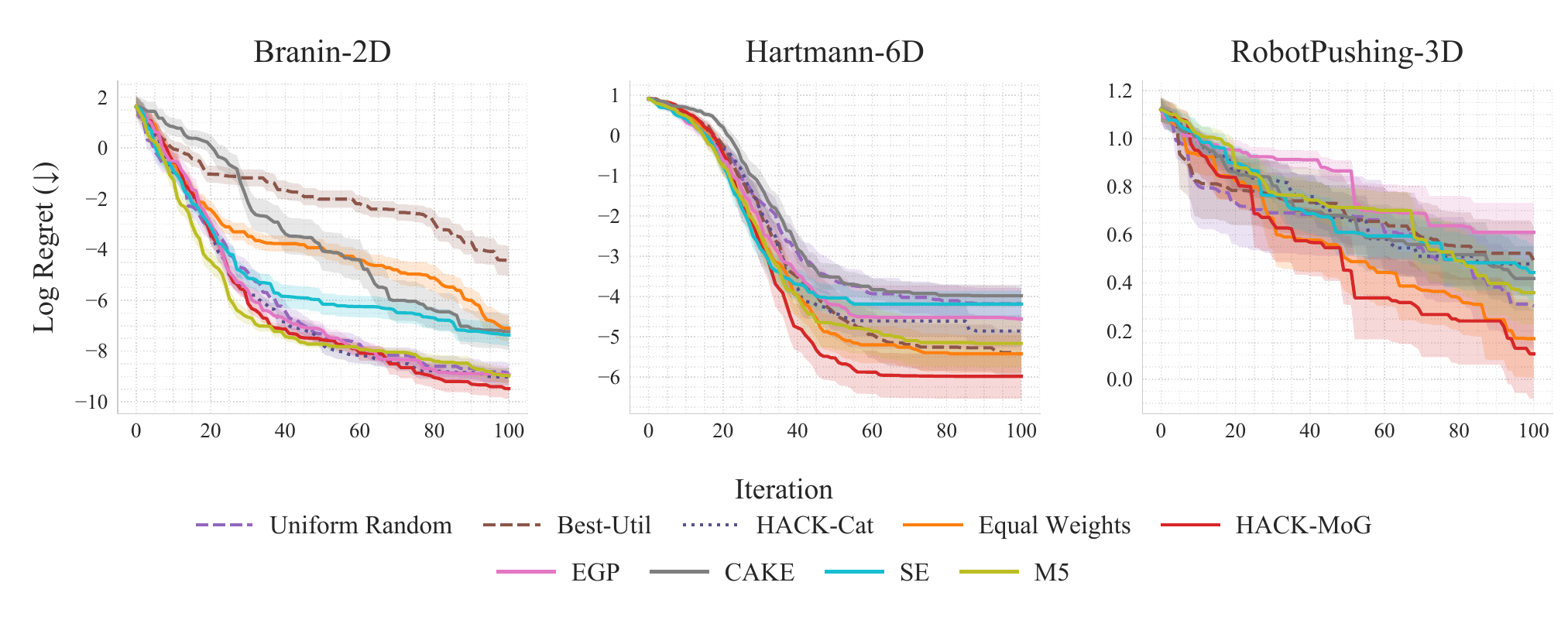}
    \caption{Log regret for the two synthetic (Branin-2D and Hartmann-6D) and the robot pushing task described in Section \ref{sec:boexp}. The lines and shaded regions represent the mean log regret and corresponding 95\% confidence interval for each method across 20 seeds. \textsc{HACK-MoG} is competitive across the three displayed tasks and achieves the best overall average rank in Table~\ref{tab:bo_synthetic}.}\label{fig:BO_performance}
\end{figure*}

\label{sec:unknown-smooth}
\subsubsection*{Unknown smoothness and adaptive weighting}
We next isolate a common source of kernel misspecification: uncertainty about the smoothness of the objective. 
For this diagnostic, we run BO on Bukin-2D using the smoothness-ordered dictionary
\[
\mathcal{K}_{\mathrm{smooth}}
=
\{\textsc{Matern-}1/2,\textsc{Matern-}3/2,\textsc{Matern-}5/2,\textsc{SE}\}.
\]
This removes other kernel families from the comparison and focuses on whether HACK can adapt among rough and smooth priors during the BO run.

Figure~\ref{fig:BO_smooth} shows that Matérn-$1/2$ is the strongest fixed kernel among this dictionary on Bukin-2D,  a function that has sharp peaks. 
Starting from uniform weights, \textsc{HACK-MoG} rapidly places most of its mass on Matérn-$1/2$, and its regret curve follows the same trend as the best fixed kernel after the early iterations. 
This example illustrates an intended use case of HACK: when the candidate kernels encode different plausible smoothness assumptions, the online loss can shift the acquisition toward the kernel that is more useful for the observed BO trajectory. 

\begin{figure*}
    \centering
    \includegraphics[width=1.0\linewidth]{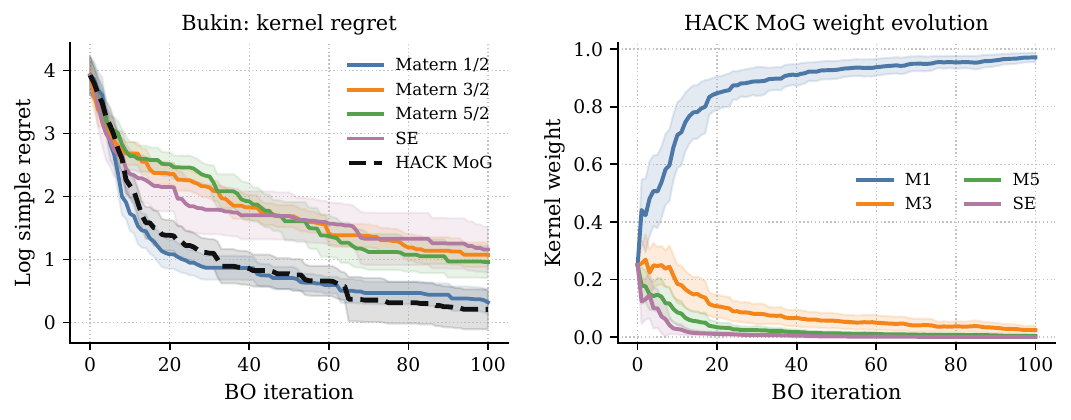}
        \caption{Unknown-smoothness diagnostic on Bukin-2D. Left: log simple regret for fixed smoothness-ordered kernels and \textsc{HACK-MoG}. The lines and shaded regions represent the mean log regret and corresponding 95\% confidence interval for each kernel/method across 20 seeds. Right: HACK kernel weights over BO iterations. The lines and shaded regions represent the mean AdaHedge weight and corresponding 95\% confidence interval for each kernel across 20 seeds. In this diagnostic, Matérn-$1/2$ performs best among the fixed kernels, and HACK rapidly shifts most of its mass to that kernel.}\label{fig:BO_smooth}
\end{figure*}

\label{sec:similar-kernels}
\subsubsection*{Similarly Performing Kernels}
When the candidate kernels induce similar BO behavior, aggressive concentration on a single expert is
neither expected nor necessarily desirable. We illustrate this regime on the Beale-2D benchmark using a
restricted kernel dictionary containing only the squared exponential and Matérn-5/2 kernels. In contrast to
the unknown-smoothness diagnostic in Figure~\ref{fig:BO_smooth}, these two kernels achieve
comparable simple-regret trajectories on this task. HACK-MoG correspondingly does not exhibit a sharp
collapse of the weight distribution onto one kernel. Instead, the AdaHedge weights remain distributed across
the two experts, reflecting the absence of a persistent loss gap large enough to justify decisive kernel
selection.

This behavior is important for interpreting HACK as an adaptive kernel-mixing method rather than a hard
model-selection procedure. When one kernel is clearly better, as in the Bukin smoothness experiment, the
weights rapidly concentrate on that kernel. When the experts are nearly tied, as in Figure~\ref{fig:beale_similar},
HACK preserves a mixture and obtains performance comparable to the fixed kernels. This is consistent with
the theory in Section~5.2: the concentration and regret-transfer guarantees are strongest under a separated
loss-gap condition, while near-tied experts naturally lead to slower concentration and a smaller practical
distinction between selecting and mixing kernels.

\begin{figure*}
    \centering
    \includegraphics[width=1.0\linewidth]{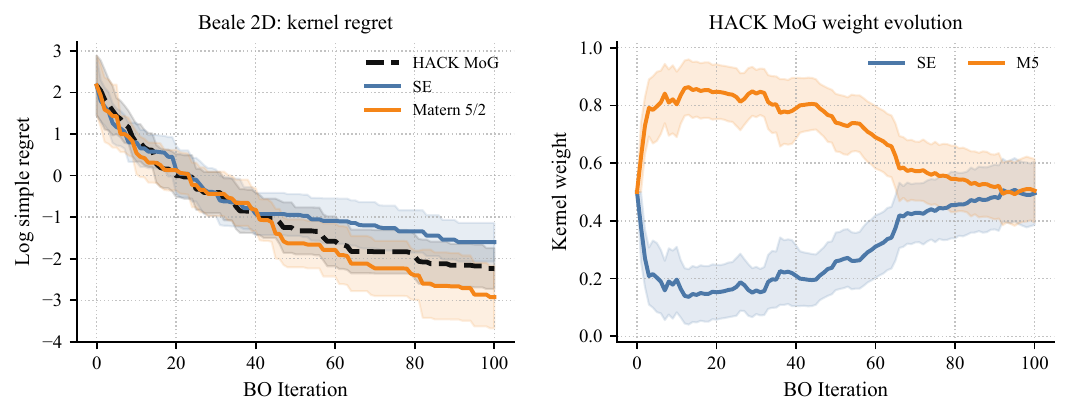}
        \caption{Similarly performing kernels on Beale-2D. Left: log simple regret for HACK-MoG compared with fixed SE and Matérn-5/2 kernels over 100 BO iterations. The lines and shaded regions represent the mean log regret and corresponding 95\% confidence interval for each method across 20 seeds. The two fixed kernels perform similarly, and HACK-MoG tracks their regret behavior without a large penalty. Right: HACK-MoG kernel weights over time. The lines and shaded regions represent the mean AdaHedge weight and corresponding 95\% confidence interval for each kernel across 20 seeds. Unlike the sharp concentration observed in the Bukin smoothness diagnostic, the weights remain shared between SE and Matérn-5/2, indicating that AdaHedge does not force premature selection when the experts have comparable task performance.}
\label{fig:beale_similar}
\end{figure*}

\begin{table*}[t]
  \caption{Final log simple regret after $T=100$ BO iterations (Lower is better) with EI averaged over 20 seeds. Mean values are shown with standard deviation in parentheses. Best performing method is in \textbf{bold}, second best is \underline{underlined}.}
\label{tab:bo_synthetic}
  \centering
  \scriptsize

  \rowcolors{3}{gray!6}{white}
  \setlength{\tabcolsep}{1.4pt}         

\begin{tabular}{l cc cc c c c cc} 
    \toprule
    \multirow{2}{*}{Test Function}
      & \multicolumn{2}{c}{Single Kernel}
      & \multicolumn{2}{c}{Adaptive}
      & \multicolumn{1}{c}{Equal} 
      & \multirow{2}{*}{EGP}
      & \multirow{2}{*}{CAKE}
      & \multicolumn{2}{c}{HACK} \\
    \cmidrule(lr){2-3} \cmidrule(lr){4-5} \cmidrule(lr){9-10}
      \rowcolor{white}
      & SE & Mat-5/2 & Random & Best Util & \multicolumn{1}{c}{Weights} 
      & & & Cat & MoG \\
    \midrule
    Branin-2D           & -7.36 (1.94) & -8.97 (1.37) & -8.88 (1.80) & -4.47 (2.70) & -7.10 (2.32) & -8.93 (1.57) & -7.23 (3.24) & \underline{-9.02 (1.54)} & \textbf{-9.48 (1.82)} \\
    Bukin-2D            &  1.52 (0.77) & \underline{ 1.06 (0.60)} &  1.45 (0.52) &  1.21 (0.83) &  1.27 (0.57) &  1.21 (0.60) &  1.56 (0.54) &  1.07 (0.81) & \textbf{ 0.88 (0.76)} \\
    Powell-4D           &  2.02 (0.98) & \textbf{ 0.89 (0.95)} &  2.90 (0.78) &  2.90 (1.32) &  2.41 (0.77) &  1.47 (1.32) &  3.26 (1.21) &  2.24 (1.13) & \underline{ 0.95 (1.24)} \\
    Ackley-5D           &  1.76 (0.32) & \underline{ 1.65 (0.24)} &  1.70 (0.31) & \textbf{ 1.44 (0.25)} &  1.66 (0.35) &  2.11 (0.61) &  1.87 (0.24) &  1.77 (0.34) &  1.69 (0.25) \\
    Hartmann-6D         & -4.19 (1.49) & -5.17 (2.72) & -4.19 (1.79) & -5.43 (2.48) & \underline{-5.43 (2.17)} & -4.57 (3.55) & -3.99 (1.19) & -4.86 (1.78) & \textbf{-5.99 (2.54)} \\
    Rastrigin-8D        &  3.99 (0.26) &  3.95 (0.17) &  3.97 (0.23) &  3.91 (0.21) & \textbf{ 3.85 (0.20)} &  3.97 (0.37) &  3.98 (0.25) &  4.25 (0.17) & \underline{ 3.90 (0.24)} \\
    Rosenbrock-8D       &  7.65 (0.71) &  7.52 (0.64) &  7.66 (0.71) &  8.87 (0.82) &  7.70 (0.66) & \textbf{ 7.08 (0.96)} &  8.86 (1.23) &  7.58 (0.80) & \underline{ 7.34 (0.70)} \\
    DixonPrice-8D       &  6.48 (0.91) &  5.49 (0.62) &  5.46 (0.86) &  5.43 (0.93) &  5.88 (0.74) & \underline{ 5.24 (1.06)} &  5.66 (1.05) &  5.46 (0.94) & \textbf{ 5.22 (1.03)} \\
    \midrule
    \textbf{Total Wins} & 0 & \underline{1} & 0 & \underline{1} & \underline{1} & \underline{1} & 0 & 0 & \textbf{4} \\
    \textbf{Avg. Rank}  & 6.62 & \underline{3.12} & 6.12 & 5.00 & 5.12 & 4.38 & 8.00 & 4.88 & \textbf{1.75} \\
    \bottomrule
  \end{tabular}
\end{table*}

\section{Level Set Estimation}
We next specialize the expert-advice framework to Level Set Estimation (LSE). In this setting, the objective is to identify the superlevel set of the unknown function relative to a fixed threshold, which naturally induces binary outcomes at queried points. We define task-specific losses that reward accurate threshold classification and analyze how expert weight concentration affects the resulting LSE acquisition function.
\subsection{Loss and Acquisition function}
For the expert loss, LSE naturally induces a binary label at the queried point:
$z_t:=\mathbf{1}\{y_t\ge h\}$. Define the predicted above-threshold probability
\[\pi^{\textsc{lse}}_{t,m}(x_t):=\mathbb{P}(Y_{t,m}\ge h\mid x_t)
=\Phi\!\big((\mu_{t,m}(x_t)-h)/\sigma_{t,m}(x_t)\big).\] We then use the Brier score
$\ell^{(t)}_{m,\textsc{brier}}=\Big(z_t-\pi^{\textsc{lse}}_{t,m}(x_t)\Big)^2,$
a proper scoring rule for probabilistic~binary predictions, naturally bounded in $[0,1]$. As in BO, we combine the task-specific loss with NLL (Equation~\ref{eq:nll_loss}) to balance threshold classification accuracy with overall predictive calibration: 
\(
\ell^{(t)}_m=\alpha\,\ell^{(t)}_{m,\textsc{brier}}+(1-\alpha)\,\ell^{(t)}_{m,\textsc{nll}}.
\)
We use EI-LSE \cite {ravishankarimprovement} as the primary acquisition function for LSE.
For the MoG predictive distribution, EI-LSE retains the same linearity-of-expectation
advantage as EI:
$\mathrm{EI\text{-}LSE}_{t,\textsc{MoG}}(x)
=\sum_{m=1}^M w_m^{(t)}\,\mathrm{EI\text{-}LSE}_{t,m}(x)\;+\;\beta\,\bar\sigma_t^2(x),$
where $\bar\sigma_t^2(x)$ is the MoG predictive variance and $\beta\ge 0$ promotes exploration. The variance-based term is included to encourage exploration near the decision boundary.

\subsection{Theoretical Analysis}

The following result specializes the generic acquisition stability argument to EI-LSE, accounting for the additional variance-based exploration term. Proofs can be found in \ref{app:lse-proofs}.
\begin{ngassumption}
\label{ass:bounded-diagonal}
Assume $\kappa^2 := \max_{m}\sup_{x\in\mathcal X} k_m(x,x) < \infty$.
\end{ngassumption}
\noindent This implies that $\sigma^2_{t,m}(x)\le k_m(x,x)\le \kappa^2$ for all $t \in \{1,\dots,T\}$, $m\in\{1,\dots,M\}$, and $x\in \mathcal X$.
\begin{lemma}
\label{thm:lse-acq-closeness}
Under Assumption~\ref{ass:bounded-diagonal}, suppose $\mathrm{EI\text{-}LSE}_{t,m}(x)\le B_{\mathrm{LSE}}$
for all $m\in\{1,\dots,M\}$ and $x\in \mathcal X$. If $w_{m^*}^\iter{t}\ge 1-\varepsilon$, then
\begin{equation*}
\sup_x \big| \mathrm{EI\text{-}LSE}_{t,\mathrm{MoG}}(x) - \mathrm{EI\text{-}LSE}_{t,m^*}(x) \big| 
\le \varepsilon B_{\mathrm{LSE}} + \beta\Big(\varepsilon \kappa^2 + \sup_x \alpha_{t,\mathrm{disag}}(x)\Big).
\end{equation*}
\end{lemma}
The bound consists of two terms. The first  reflects expert weight concentration, as in the BO setting. The second one arises from the additional variance-based exploration component in EI-LSE and captures residual disagreement between experts. As weights concentrate and predictive means align, this term becomes small. Beyond acquisition-level stability, weight concentration also has a direct structural implication for the inferred level set. Intuitively, once the mixture-of-Gaussians behaves similarly to the best expert, the superlevel sets they induce can differ only near the decision boundary, where posterior uncertainty is inherently high. The following corollary makes this intuition precise by localizing the discrepancy between the MoG and best-expert level sets to a narrow boundary region.

\begin{corollary}
\label{cor:lse-boundary-band}
For $\tau\in(0,1)$, define the set $S_{t,m}(\tau):=\{x:\ \pi^{\textsc{lse}}_{t,m}(x)\ge \tau\}$ and define $S_{t,\mathrm{MoG}}(\tau)$ similarly. If $w_{m^*}^{(t)}\ge 1-\varepsilon$, then the symmetric difference of the two sets satisfies  
\(
S_{t,\mathrm{MoG}}(\tau)\triangle S_{t,m^*}(\tau)
\subseteq
\left\{
x\in\mathcal X:
\left|
\pi^{\textsc{lse}}_{t,m^*}(x)-\tau
\right|
\le\varepsilon
\right\}.\)
\end{corollary}

\subsection{Experiments} \label{lse:experiments}
Our evaluations consist of two synthetic test functions, Himmelblau and GaussianModulatedSinusoid (GMS) (Details in \ref{app:exp_dets}), and the motorcycle dataset from \cite{silverman1985some}. The synthetic functions are evaluated on a $50 \times 50$ grid uniformly and motorcycle is evaluated on 133 uniform points from their bounds. We use the mean of the posterior model to classify all the points in the grid as above/below the threshold. For HACK,  $\bar{\mu}_t$ is used as the mean. EI-LSE is run for $T=50$ iterations and results are averaged over 30 seeds. 
Because LSE can be framed as predicting whether a function value exceeds a given threshold, the F1-score serves as a natural evaluation metric. We compare against 4 baselines: (i) EI-LSE with SE and Mat\'ern-5/2, (ii) RMILE \cite{zanette2018robust} with SE and Mat\'ern-5/2, (iii) Equally weighted MoG, (iv) Random Uniform selection. All ensemble methods use the same set of kernels $\mathcal{K}=\{\textsc{SE},\textsc{Matern-}5/2,\textsc{Matern-}3/2,\textsc{RQ},\textsc{LIN},\textsc{PER}\}$.

Uniform Random and Equal Weights exhibit significantly lower performance in the Himmelblau task, even compared to standard single kernels like \textsc{SE} and \textsc{Matern-}5/2. While most methods reach comparable F1 scores by iteration 50, HACK and RMILE effectively address model misspecification, particularly within the SE kernel on the GMS function. On the Motorcycle dataset, RMILE performs slightly worse than HACK, where SE kernel misspecification is again observed. Overall, we note that HACK is able to match the performance of standard kernels while remaining robust to kernel misspecification. RMILE provides a strong baseline as a non-kernel-based robust LSE method, yet HACK performs comparably or better across tasks. We note that the per-iteration runtime of RMILE is $O(|\mathcal{X}|^2)$ \cite{zanette2018robust}, compared to $O(|\mathcal{X}|)$ for standard acquisition functions such as 
EI-LSE used by HACK.

\begin{figure*}[t]
    \centering
    \includegraphics[width=1.0\linewidth]{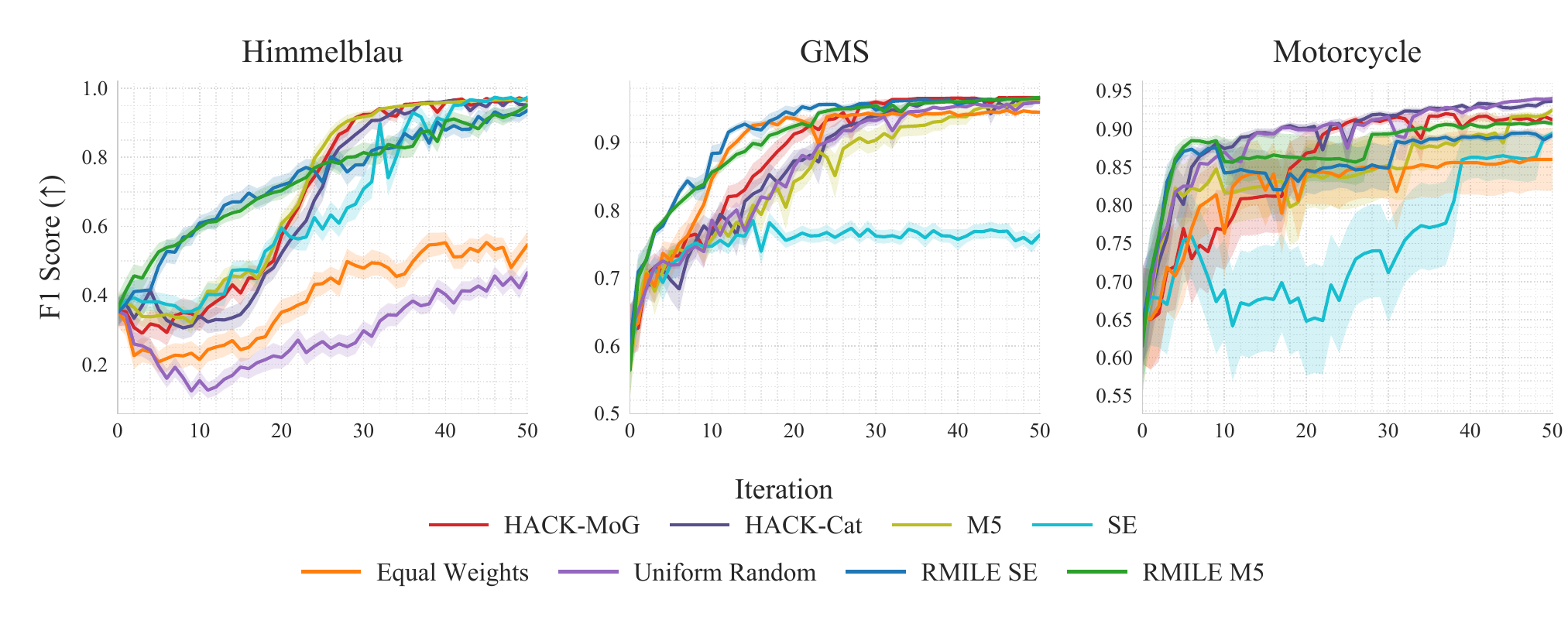}
    \caption{F1-scores for the 2 synthetic (Himmelblau and GMS) and 1 real-world data task (Motorcycle) described in \ref{lse:experiments}. The lines and shaded regions represent the mean F1 score and corresponding 95\% confidence interval for each method across 30 seeds. We see that HACK achieves similar to better performance compared to SE and Matern 5/2. We note the apparent misspecification of SE in the GMS and to a lesser extent in Motorcycle. Equal Weights and Uniform Random show the necessity of AdaHedge in rebalancing the weights of the kernels effectively.}
    \label{fig:LSE_performance}
\end{figure*}

\section{Bayesian Active Learning}
We finally specialize the framework to Bayesian Active Learning (BAL), 
where the goal is to sequentially select informative observations so as to improve the predictive model of the unknown function. In contrast to BO, which focuses on locating an optimum, or LSE, which focuses on identifying a thresholded region, BAL aims to reduce predictive uncertainty over the input space. Unlike BO and LSE, it does not naturally provide a binary task-progress event analogous to improvement or threshold crossing.
The acquisition itself is uncertainty reduction, so we use predictive variance directly and score experts primarily through predictive calibration. In this setting, the MoG variance decomposition is especially interpretable because the acquisition separates within-expert uncertainty from between-kernel disagreement.
\subsection{Loss and Acquisition Function}

In BAL for regression, a common acquisition strategy is to query points with high posterior uncertainty. Accordingly, we use predictive variance as the acquisition function in our experiments. For the MoG predictive distribution, this acquisition reduces exactly to the mixture variance $\bar{\sigma}^2_t(x)$, which decomposes into within-expert variance and between-expert disagreement as in Lemma~\ref{lemma:meanvar}. For expert weighting, we use the normalized NLL loss (Equation \ref{eq:nll_loss}). Unlike BO and LSE, BAL does not naturally provide a binary task-specific event for defining a Brier-style loss.

\subsection{Theoretical Analysis}
We now show that expert weight concentration implies uniform closeness of the BAL acquisition function under the MoG predictive distribution. Proofs can be found in \ref{app:bal-proofs}.
\begin{lemma}
\label{lemma:bal-closeness}
Assume Assumption~\ref{ass:bounded-diagonal}. If $w_{m^*}^{\iter{t}} \ge 1-\varepsilon$, then
\[
\sup_{x\in \mathcal{X}}
\big|\sigma^2_{t,\mathrm{MoG}}(x)-\sigma^2_{t,m^*}(x)\big|
\le
\varepsilon\kappa^2 + \sup_{x\in\mathcal X}\alpha_{t,\mathrm{disag}}(x).
\]
\end{lemma}

Lemma~\ref{lemma:bal-closeness} shows that once the expert weights concentrate, the predictive variance under the MoG closely tracks that of the best expert, up to residual disagreement across experts. In particular, when expert predictions align, BAL queries selected using the MoG acquisition behave similarly to those selected using the best kernel.
In Bayesian Active Learning, the variance-based acquisition admits a particularly transparent decomposition under the MoG predictive distribution. Expert weight concentration suppresses disagreement-driven uncertainty, causing the MoG acquisition to approximate that of the best expert while retaining robustness to kernel misspecification in earlier rounds.

\begin{figure}[H]
    \centering
    \includegraphics[width=1.0\linewidth]{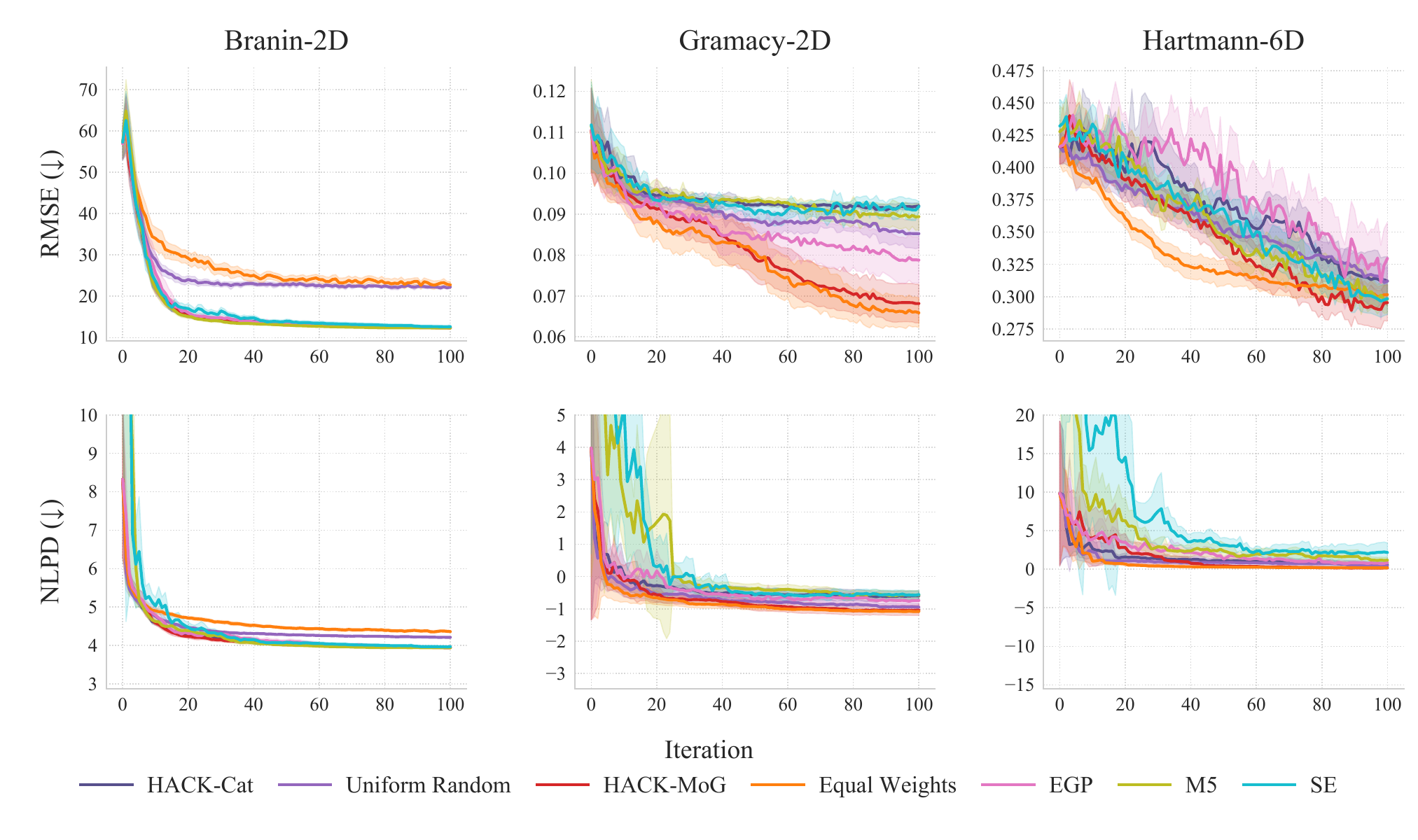}
    \caption{RMSE and NLPD for the 3 BAL synthetic test functions. The lines and shaded regions represent the mean RMSE (Top) and NLPD (Bottom) score and corresponding 95\% confidence interval for each method across 30 seeds.}\label{fig:BAL_performance}
\end{figure}

\subsection{Experiments}
We evaluate on three synthetic functions (Branin-2D, Gramacy-2D, and Hartmann-6D) with the same settings as \cite{riis2022bayesian}. Predictive performance is evaluated using Root Mean Squared Error (RMSE) and Negative Log Predictive Density (NLPD). RMSE measures the accuracy of the posterior predictive mean, while NLPD evaluates the probability assigned by the predictive distribution to the observed values and therefore also reflects the quality of its uncertainty estimates. Lower values are better for both metrics. Reporting both is useful in BAL because accurate mean predictions alone do not guarantee well-calibrated uncertainty, which directly influences where the model chooses to query next.

We run BAL for $T=100$ iterations and results are averaged over 30 seeds. All ensemble methods use the same set of kernels $\mathcal{K}=\{\textsc{SE},\textsc{Matern-}5/2,\textsc{Matern-}3/2,\textsc{RQ},\textsc{LIN},\textsc{PER}\}$. We compare against: (i) fixed kernels {SE} and Matérn-$5/2$; (ii) Uniform Random; (iii) Equally weighted MoG; and (iv) EGP with Variance of GP Mixtures from \cite{EnsemblesBAL}. 
HACK performs better than EGP on Gramacy-2D and Hartmann-6D, while remaining competitive on Branin-2D. Overall, the results support the robustness of adaptive kernel weighting across problems with varying best kernels, rather than a consistent large improvement on every benchmark.

\section{Conclusion}
We provide a framework, HACK, an online mechanism for adapting over a finite set of GP kernels during sequential decision-making. Rather than committing to a single kernel before data collection, it uses task-aware losses to reweight candidate GP experts as iterations pass. The method is most beneficial when kernel misspecification would otherwise cause poor BO, LSE, or BAL decisions, and it reduces to conservative model averaging when several kernels are similarly useful. Its guarantees and empirical behavior are therefore best understood through the size of the loss gap between candidate kernels: large gaps lead to concentration and best-expert-like behavior, while small gaps lead to shared weights without necessarily harming task~performance.

\section*{Acknowledgments}
This work was supported by the ANRF MATRICS project grant number MTR/2023/000042 and the AlphaGrep Quantitative Research lab at IIIT-H. Most of this work was done while GD was on (sabbatical) leave from ASU, visiting IIIT-H.

\bibliographystyle{elsarticle-num}
\bibliography{references}

\newpage
\appendix

\section{Proofs}
\subsection{} \label{app:A1}

We first provide proofs of the weight-concentration results with AdaHedge and Standard Hedge used as the online learning algorithms.

\paragraph{Lemma \ref{lemma:kernelconc}}
Assume the per-round loss difference between the best kernel $m^*$ and any other kernel is bounded from below,
$\ell_m^{(t)} - \ell_{m^*}^{(t)} \geq \Delta > 0$.
Then $w_{m^*}^\iter{T} \geq 1-\varepsilon$ under either of the following conditions:

\begin{enumerate}
    \item with $\eta^\iter{T}$ given by Equation \ref{eq:adahedgelr} (AdaHedge),  
    \[
    T \geq 1 +\max \left\{
    \frac{4a_M^2}{\ln M},
    \frac{1}{\Delta^2\ln M }
    \ln^2\left(\frac{(M-1)(1-\varepsilon)}{\varepsilon}\right)
    \right\},
    \]
    where $a_M = 1 + \frac{2}{3}\ln M$.

    \item with a constant $\eta^\iter{T}=\eta$ (Hedge),  
    \[
    T\geq 1 + \frac{1}{\eta\Delta}
    \ln\left(\frac{(M-1)(1-\varepsilon)}{\varepsilon}\right).
    \]
\end{enumerate}

\begin{proof}
For both update rules,
\begin{align*}
w_{m^*}^\iter{T}
&= \frac{\exp\left(-\eta^\iter{T}L_{m^*}^\iter{T-1}\right)}
{\sum_{i=1}^M\exp\left(-\eta^\iter{T}L_i^\iter{T-1}\right)} \\
&= \frac{1}{1 +\sum_{i=1, i\neq {m^*}}^M
\exp\left(-\eta^\iter{T}(L_i^\iter{T-1} - L_{m^*}^\iter{T-1}) \right)}\\
&\geq \frac{1}{1 +(M-1)
\exp\left(-\eta^\iter{T}(T-1)\Delta \right)}.
\end{align*}

We will have $w_{m^*}^\iter{T} \geq 1 - \varepsilon$ if
\begin{align}
\frac{1}{1 +(M-1)\exp\left(-\eta^\iter{T}(T-1)\Delta \right)}
&\geq 1-\varepsilon \notag \\
\eta^\iter{T}(T-1)\Delta
&\geq
\ln\left(\frac{(M-1)(1-\varepsilon)}{\varepsilon}\right).
\label{eq:T_bound_1_adahedge}
\end{align}

\noindent\textbf{Case 1: AdaHedge.}

From De Rooij et al.~\cite[Lemma~5]{adahedge},
\[
(\Delta_{mix}^\iter{T-1})^2
\leq
V_{T-1}\ln M
+
\left(1 + \frac{2}{3}\ln M\right)
\Delta_{mix}^\iter{T-1}.
\]
From De Rooij et al.~\cite[Lemma~4]{adahedge}, $v_t\leq1/4$, thus
$V_{T-1} \leq \frac{T-1}{4}$.
Using $x^2 \leq a +bx \implies x \leq \sqrt{a} + b$,
\[
\Delta_{mix}^\iter{T-1}
\leq
\frac{1}{2}\sqrt{(T-1)\ln M}
+
1 + \frac{2}{3}\ln M.
\]

\[
\eta^\iter{T}
= \frac{\ln M}{\Delta_{mix}^\iter{T-1}} \geq
\frac{2\ln M}
{\sqrt{(T-1)\ln M}
+ 2(1 + \frac{2}{3}\ln M)}.
\label{eq:eta_ineq_adahedge}
\]

Let $a_M:=1 +\frac{2}{3}\ln M$.
If $\sqrt{(T-1)\ln M} \geq 2a_M$, then
\[
T \geq 1 + \frac{4 a_M^2}{\ln M}
\]
and
\[
\eta^\iter{T} \geq \sqrt{\frac{\ln M}{T -1}}.
\]
Substituting this into
\eqref{eq:T_bound_1_adahedge} gives
\[
T \geq
1 + \frac{1}{\Delta^2\ln M }
\ln^2\left(\frac{(M-1)(1-\varepsilon)}{\varepsilon}\right).
\]

Hence the stated AdaHedge condition is sufficient.

\medskip
\noindent\textbf{Case 2: Hedge.}

For Hedge, $\eta^\iter{T}=\eta$ is fixed.
Substituting this into
\eqref{eq:T_bound_1_adahedge} gives
\[
T\geq
1 + \frac{1}{\eta\Delta}
\ln\left(\frac{(M-1)(1-\varepsilon)}{\varepsilon}\right).
\]

Hence the stated Hedge condition is sufficient.
\end{proof}

\paragraph{Lemma \ref{lemma:kernelconcexpectation}}
Let $\mathcal{F}_t$ be the filtration generated by all randomness and observations up to round $t$.
Assume the per-round loss difference between the best kernel $m^*$ and any other kernel is bounded from below in expectation,
\[
\mathbb{E}\left[
\ell_m^{(t)} - \ell_{m^*}^{(t)}
\mid \mathcal{F}_{t-1}
\right]
\geq \Delta,
\qquad
\forall m\neq m^*,\; t\in\{1,\dots,T-1\}.
\]
Then $w_{m^*}^\iter{T} \geq 1-\varepsilon$ with probability at least $1-\delta$
under either of the following conditions:

\begin{enumerate}
    \item with $\eta^\iter{T}$ given by Equation \ref{eq:adahedgelr} (AdaHedge),
    \[
    T \geq 1 +\max \left\{
    \frac{4a_M^2}{\ln M},
    \frac{4}{\Delta^2\ln M}
    \ln^2\left(\frac{(M-1)(1-\varepsilon)}{\varepsilon}\right),
    \frac{32}{\Delta^2}
    \ln\left(\frac{M-1}{\delta}\right)
    \right\},
    \]
    where $a_M = 1 + \frac{2}{3}\ln M$.

    \item with a constant $\eta^\iter{T}=\eta$ (Hedge), 
    \[
    T \geq 1 +\max\left\{
    \frac{2}{\eta\Delta}
    \ln\left(\frac{(M-1)(1-\varepsilon)}{\varepsilon}\right),
    \frac{32}{\Delta^2}
    \ln\left(\frac{M-1}{\delta}\right)
    \right\}.
    \]
\end{enumerate}

\begin{proof}
Let
\[
X_i^\iter{t}
:=
(\ell_i^\iter{t}-\ell_{m^*}^\iter{t})
-
\mathbb{E}\left[
\ell_i^\iter{t}-\ell_{m^*}^\iter{t}
\mid \mathcal{F}_{t-1}
\right].
\]
Then $\{X_i^\iter{t}\}$ is a martingale difference sequence.
Using Azuma--Hoeffding and noting that
$(\ell_i^\iter{t}-\ell_j^\iter{t})\in[-1,1]$, which implies
$|X_i^\iter{t}|\leq 2$, we have
\begin{equation*}
\mathbb{P}\left(
\sum_{t=1}^{T-1} X_i^\iter{t} \leq -a
\right)
\leq
\exp\left(-\frac{a^2}{8(T-1)}\right).
\end{equation*}

Using $
\mathbb{E}\left[
\ell_i^\iter{t}-\ell_{m^*}^\iter{t}
\mid \mathcal{F}_{t-1}
\right]
\geq \Delta,$
we obtain
\begin{equation*}
\begin{aligned}
\mathbb{P}\left(
L_i^\iter{T-1} - L_{m^*}^\iter{T-1}
\leq -a + (T-1)\Delta
\right)
&\leq
\mathbb{P}\Bigg(
L_i^\iter{T-1} - L_{m^*}^\iter{T-1}
\leq
-a + \sum_{t=1}^{T-1}
\mathbb{E}\left[
\ell_i^\iter{t}-\ell_{m^*}^\iter{t}
\mid \mathcal{F}_{t-1}
\right]
\Bigg) \\
&\leq
\exp\left(-\frac{a^2}{8(T-1)}\right).
\end{aligned}
\end{equation*}
Setting
$
a =
2\sqrt{2(T-1)\ln\left(\frac{M-1}{\delta}\right)}$
gives
\[
\mathbb{P}\left(
L_i^\iter{T-1} - L_{m^*}^\iter{T-1}
\leq
(T-1)\Delta
-
2\sqrt{2(T-1)\ln\left(\frac{M-1}{\delta}\right)}
\right)
\leq
\frac{\delta}{M-1}.
\]
Using a union bound over all $i\neq m^*$, with probability at least $1-\delta$,
\[
L_i^\iter{T-1} - L_{m^*}^\iter{T-1}
\geq
(T-1)\Delta
-
2\sqrt{2(T-1)\ln\left(\frac{M-1}{\delta}\right)}
\]
for every $i\neq m^*$. If
$T
\geq
1+\frac{32}{\Delta^2}
\ln\left(\frac{M-1}{\delta}\right),
$
then
\[
L_i^\iter{T-1} - L_{m^*}^\iter{T-1}
\geq
\frac{1}{2}(T-1)\Delta
\]
for every $i\neq m^*$. Therefore,
\begin{align*}
w_{m^*}^\iter{T}
&=
\frac{1}{
1+\sum_{i=1,i\neq m^*}^M
\exp\left(
-\eta^\iter{T}
(L_i^\iter{T-1}-L_{m^*}^\iter{T-1})
\right)
} \\
&\geq
\frac{1}{
1+(M-1)
\exp\left(
-\eta^\iter{T}\frac{(T-1)\Delta}{2}
\right)
}.
\end{align*} We will have $w_{m^*}^\iter{T}\geq 1-\varepsilon$ if
\begin{align}
\frac{1}{
1+(M-1)
\exp\left(
-\eta^\iter{T}\frac{(T-1)\Delta}{2}
\right)}
&\geq
1-\varepsilon, \notag \\
\eta^\iter{T}\frac{(T-1)\Delta}{2}
&\geq
\ln\left(
\frac{(M-1)(1-\varepsilon)}{\varepsilon}
\right).
\label{eq:T_bound_2_adahedge}
\end{align}

\noindent\textbf{Case 1: AdaHedge.}

As shown in the proof of Lemma \ref{lemma:kernelconc}, if
\(
T \geq 1 +\frac{4a_M^2}{\ln M},
\)
then
\[
\eta^\iter{T} \geq \sqrt{\frac{\ln M}{T -1}}.
\]
Substituting this into Inequality
\ref{eq:T_bound_2_adahedge}, we require
\[
\frac{\Delta}{2}\sqrt{(T-1)\ln M}
\geq
\ln\left(\frac{(M-1)(1-\varepsilon)}{\varepsilon}\right),
\]
which gives
\[
T \geq 1 + \frac{4}{\Delta^2\ln M }
\ln^2\left(\frac{(M-1)(1-\varepsilon)}{\varepsilon}\right).
\]
Together with
\[
T \geq 1 +\frac{32}{\Delta^2}
\ln\left(\frac{M-1}{\delta}\right),
\]
the AdaHedge condition follows.

\noindent\textbf{Case 2: Hedge.}

For Hedge, $\eta^\iter{T}=\eta$. Therefore, Inequality
\ref{eq:T_bound_2_adahedge} becomes
\[
\eta\frac{(T-1)\Delta}{2}
\geq
\ln\left(\frac{(M-1)(1-\varepsilon)}{\varepsilon}\right),
\]
which gives
\[
T \geq
1+\frac{2}{\eta\Delta}
\ln\left(\frac{(M-1)(1-\varepsilon)}{\varepsilon}\right).
\]
Together with
\[
T \geq 1+\frac{32}{\Delta^2}
\ln\left(\frac{M-1}{\delta}\right),
\]
the Hedge condition follows. \qedhere
\end{proof}

\subsection{}

\paragraph{Corollary \ref{cor:error-rate}}
Assume the conditional loss-gap condition of Lemma \ref{lemma:kernelconcexpectation} holds for \(m^*\) with gap \(\Delta > 0\).
Fix a horizon \(T\), let \(\delta_0 = \delta_{\mathrm{conc}}/n_T\), and define
\[
\epsilon_t
=
\frac{M-1}{M-1+\exp\!\left(\frac{\Delta}{2}\sqrt{(t-1)\ln M}\right)},
\]
for $t\in I_T$ and let $T_0$ satisfy
\[
T_0
\ge
\max\left\{
\frac{4a_M^2}{\ln M},\;
\frac{32}{\Delta^2}\ln\!\left(\frac{M-1}{\delta_0}\right)
\right\}.
\]
Then, for this fixed horizon, with $c=\tfrac{\Delta}{2}\sqrt{\ln M}$, the summation of $\epsilon_t$ post-burn-in is bounded by
\[
\sum_{t=T_0+1}^{T}\epsilon_t
\;\le\;
(M-1)\!\left[
1 + \frac{2}{c^2}\!\left(1+c\sqrt{T_0}\right)
\right]e^{-c\sqrt{T_0}}.
\]
Consequently, in
Theorem~\ref{thm:hack-ei-regret-transfer} the additional simple-regret term satisfies
\[
\frac{2B_{\mathrm{EI}}}{n_T}
\sum_{t=T_0+1}^{T}\epsilon_t
=
O\!\left(\frac{1}{n_T}\right).
\]
\begin{proof}
Fix $t\in I_T$ and apply Lemma \ref{lemma:kernelconcexpectation} with terminal time $t$, confidence level
$\delta_0$, and error level $\epsilon_t$. Since $t-1\ge T_0$, the first and
third terms in the maximum of Lemma \ref{lemma:kernelconcexpectation} are bounded by $t-1$ by assumption.
For the second term, the definition of $\epsilon_t$ gives
\[
\ln\left(
\frac{(M-1)(1-\epsilon_t)}{\epsilon_t}
\right)
=
\frac{\Delta}{2}\sqrt{(t-1)\ln M}.
\]
Hence
\[
\frac{4}{\Delta^2\ln M}
\ln^2\left(
\frac{(M-1)(1-\epsilon_t)}{\epsilon_t}
\right)
=
t-1.
\]
Thus Lemma \ref{lemma:kernelconcexpectation} yields
\(
\mathbb P\left(w_{m^*}^{(t)}\ge 1-\epsilon_t\right)
\ge
1-\delta_0 .
\)
Taking a union bound over $t\in I_T$ gives the stated simultaneous event.
Finally,
\[
\epsilon_t
=
\frac{M-1}{M-1+e^{c\sqrt{t-1}}}
\le
(M-1)e^{-c\sqrt{t-1}},
\]
so
\[
\sum_{t=T_0+1}^{T}\epsilon_t
\le
(M-1)\sum_{s=T_0}^{T-1}e^{-c\sqrt{s}}
\le
(M-1)\left(
e^{-c\sqrt{T_0}}
+
\int_{T_0}^{\infty}e^{-c\sqrt{x}}\,dx
\right).
\]
The change of variables $u=\sqrt{x}$ gives
\[
\int_{T_0}^{\infty}e^{-c\sqrt{x}}\,dx
=
2\int_{\sqrt{T_0}}^{\infty}u e^{-cu}\,du
=
\frac{2}{c^2}(1+c\sqrt{T_0})e^{-c\sqrt{T_0}},
\]
which proves the summation bound.
\end{proof}

\subsection{}
\label{app:lse-proofs}

\paragraph{Lemma \ref{thm:lse-acq-closeness}}
Under Assumption~\ref{ass:bounded-diagonal}, suppose $\mathrm{EI\text{-}LSE}_{t,m}(x)\le B_{\mathrm{LSE}}$
for all $m\in\{1,\dots,M\}$ and $x\in \mathcal X$. If $w_{m^*}^\iter{t}\ge 1-\varepsilon$, then
\begin{equation*}
\sup_x \big| \mathrm{EI\text{-}LSE}_{t,\mathrm{MoG}}(x) - \mathrm{EI\text{-}LSE}_{t,m^*}(x) \big| 
\le \varepsilon B_{\mathrm{LSE}} + \beta\Big(\varepsilon \kappa^2 + \sup_x \alpha_{t,\mathrm{disag}}(x)\Big).
\end{equation*}
\begin{proof}
By the MoG construction (linearity of expectation for the EI--LSE term),
\[
\mathrm{EI\text{-}LSE}_{t,\mathrm{MoG}}(x)
=\sum_{m=1}^M w_m^{(t)}\,\mathrm{EI\text{-}LSE}_{t,m}(x)\;+\;\beta\,\bar\sigma_t^2(x),
\]
where $\bar\sigma_t^2(x)$ is the MoG posterior predictive variance.  Let
$m^*\in\{1,\dots,M\}$ satisfy $w_{m^*}^{(t)}\ge 1-\varepsilon$.
Fix any $x\in\mathcal X$,
\begin{align*}
\mathrm{EI\text{-}LSE}_{t,\mathrm{MoG}}(x)-\mathrm{EI\text{-}LSE}_{t,m^*}(x)
&=\sum_{m=1}^M w_m^{(t)}\,\mathrm{EI\text{-}LSE}_{t,m}(x)-\mathrm{EI\text{-}LSE}_{t,m^*}(x)
\;+\;\beta\,\bar\sigma_t^2(x)\\
&=\sum_{m\neq m^*} w_m^{(t)}\Big(\mathrm{EI\text{-}LSE}_{t,m}(x)-\mathrm{EI\text{-}LSE}_{t,m^*}(x)\Big)
\;+\;\beta\,\bar\sigma_t^2(x).
\end{align*}
Taking absolute values and applying the triangle inequality yields
\begin{align*}
\big|\mathrm{EI\text{-}LSE}_{t,\mathrm{MoG}}(x)-\mathrm{EI\text{-}LSE}_{t,m^*}(x)\big|
&\le
\sum_{m\neq m^*} w_m^{(t)}\,
\big|\mathrm{EI\text{-}LSE}_{t,m}(x)-\mathrm{EI\text{-}LSE}_{t,m^*}(x)\big|
\;+\;\beta\,\bar\sigma_t^2(x).
\end{align*}
Since $0\le \mathrm{EI\text{-}LSE}_{t,m}(x)\le B_{\mathrm{LSE}}$ for all $m$ and $x$,
\(
\big|\mathrm{EI\text{-}LSE}_{t,m}(x)-\mathrm{EI\text{-}LSE}_{t,m^*}(x)\big|\le B_{\mathrm{LSE}},
\)
and therefore
\[
\sum_{m\neq m^*} w_m^{(t)}\,
\big|\mathrm{EI\text{-}LSE}_{t,m}(x)-\mathrm{EI\text{-}LSE}_{t,m^*}(x)\big|
\le \sum_{m\neq m^*} w_m^{(t)}\,B_{\mathrm{LSE}}
=(1-w_{m^*}^{(t)})B_{\mathrm{LSE}}
\le \varepsilon B_{\mathrm{LSE}}.
\]
Next, use the MoG posterior predictive variance decomposition
\[
\bar\sigma_t^2(x)=\sum_{m=1}^M w_m^{(t)}\,\sigma_{t,m}^2(x)\;+\;\alpha_{t,\mathrm{disag}}(x),
\]
Hence,
\[
\bar\sigma_t^2(x)
=\sum_{m\neq m^*} w_m^{(t)}\,\sigma_{t,m}^2(x) \;+\; w_{m^*}^{(t)}\sigma_{t,m^*}^2(x)
\;+\;\alpha_{t,\mathrm{disag}}(x)
\le (1-w_{m^*}^{(t)})\kappa^2 \;+\;\sigma_{t,m^*}^2(x)\;+\;\alpha_{t,\mathrm{disag}}(x),
\]
where the inequality uses Assumption \ref{ass:bounded-diagonal}: $\sigma_{t,m}^2(x)\le \kappa^2$ for all $m,x$.
Combining the two bounds gives, for every $x$,
\[
\big|\mathrm{EI\text{-}LSE}_{t,\mathrm{MoG}}(x)-\mathrm{EI\text{-}LSE}_{t,m^*}(x)\big|
\le \varepsilon B_{\mathrm{LSE}}
\;+\;\beta\Big(\varepsilon\kappa^2+\alpha_{t,\mathrm{disag}}(x)\Big).
\]
Taking $\sup_{x\in\mathcal X}$ finishes the proof:
\[
\sup_{x\in\mathcal X}
\big|\mathrm{EI\text{-}LSE}_{t,\mathrm{MoG}}(x)-\mathrm{EI\text{-}LSE}_{t,m^*}(x)\big|
\le \varepsilon B_{\mathrm{LSE}}
\;+\;\beta\Big(\varepsilon\kappa^2+\sup_{x\in\mathcal X}\alpha_{t,\mathrm{disag}}(x)\Big). \qedhere
\]
\end{proof}

\paragraph{Corollary \ref{cor:lse-boundary-band}}
For $\tau\in(0,1)$, define the set $S_{t,m}(\tau):=\{x:\ \pi^{\textsc{lse}}_{t,m}(x)\ge \tau\}$ and define $S_{t,\mathrm{MoG}}(\tau)$ similarly. If $w_{m^*}^{(t)}\ge 1-\varepsilon$, then the symmetric difference of the two sets satisfies  
\[
S_{t,\mathrm{MoG}}(\tau)\triangle S_{t,m^*}(\tau)
\subseteq
\left\{
x\in\mathcal X:
\left|
\pi^{\textsc{lse}}_{t,m^*}(x)-\tau
\right|
\le\varepsilon
\right\}.\]
\begin{proof}
Under the MoG predictive distribution,
the law of total probability gives, for every $x\in\mathcal X$,
\[
\pi^{\textsc{lse}}_{t,\mathrm{MoG}}(x)
:=\mathbb{P}_{\mathrm{MoG}}(Y_t\ge h\mid x, D_{t-1})
=\sum_{m=1}^M w_m^{(t)}\,\pi^{\textsc{lse}}_{t,m}(x).
\]
Fix $m^*$ such that $w_{m^*}^{(t)}\ge 1-\varepsilon$, and fix any $x\in\mathcal X$.
Then
\begin{align*}
\pi^{\textsc{lse}}_{t,\mathrm{MoG}}(x)-\pi^{\textsc{lse}}_{t,m^*}(x)
&=\sum_{m=1}^M w_m^{(t)}\,\pi^{\textsc{lse}}_{t,m}(x)-\pi^{\textsc{lse}}_{t,m^*}(x)\\
&=\sum_{m\neq m^*} w_m^{(t)}\Big(\pi^{\textsc{lse}}_{t,m}(x)-\pi^{\textsc{lse}}_{t,m^*}(x)\Big).
\end{align*}
Taking absolute values and applying the triangle inequality yields
\[
\Big|\pi^{\textsc{lse}}_{t,\mathrm{MoG}}(x)-\pi^{\textsc{lse}}_{t,m^*}(x)\Big|
\le \sum_{m\neq m^*} w_m^{(t)}\,
\Big|\pi^{\textsc{lse}}_{t,m}(x)-\pi^{\textsc{lse}}_{t,m^*}(x)\Big|.
\]
Since each $\pi^{\textsc{lse}}_{t,m}(x)\in[0,1]$, we have
$\big|\pi^{\textsc{lse}}_{t,m}(x)-\pi^{\textsc{lse}}_{t,m^*}(x)\big|\le 1$, hence
\[
\Big|\pi^{\textsc{lse}}_{t,\mathrm{MoG}}(x)-\pi^{\textsc{lse}}_{t,m^*}(x)\Big|
\le \sum_{m\neq m^*} w_m^{(t)}
=1-w_{m^*}^{(t)}
\le \varepsilon.
\]
Let $x\in S_{t,\mathrm{MoG}}(\tau)\setminus S_{t,m^*}(\tau)$.
Then $\pi^{\textsc{lse}}_{t,\mathrm{MoG}}(x)\ge \tau$ and $\pi^{\textsc{lse}}_{t,m^*}(x)<\tau$, so
\[
\tau \le \pi^{\textsc{lse}}_{t,\mathrm{MoG}}(x)
\le \pi^{\textsc{lse}}_{t,m^*}(x)+\varepsilon
\quad\Rightarrow\quad
\pi^{\textsc{lse}}_{t,m^*}(x)\ge \tau-\varepsilon
\quad\Rightarrow\quad
\big|\pi^{\textsc{lse}}_{t,m^*}(x)-\tau\big|\le \varepsilon.
\]
Similarly, if $x\in S_{t,m^*}(\tau)\setminus S_{t,\mathrm{MoG}}(\tau)$,
then $\pi^{\textsc{lse}}_{t,m^*}(x)\ge \tau$ and $\pi^{\textsc{lse}}_{t,\mathrm{MoG}}(x)<\tau$, so
\[
\pi^{\textsc{lse}}_{t,m^*}(x)
\le \pi^{\textsc{lse}}_{t,\mathrm{MoG}}(x)+\varepsilon
< \tau+\varepsilon
\quad\Rightarrow\quad
\big|\pi^{\textsc{lse}}_{t,m^*}(x)-\tau\big|\le \varepsilon.
\]
Combining the two cases,
\[
S^{\textsc{lse}}_{t,\mathrm{MoG}}(\tau)\,\triangle\, S^{\textsc{lse}}_{t,m^*}(\tau)
\subseteq
\left\{x\in\mathcal X:\big|\pi^{\textsc{lse}}_{t,m^*}(x)-\tau\big|\le \varepsilon\right\},
\]
which is the claimed inclusion.
\end{proof}

\subsection{}
\label{app:bal-proofs}

\paragraph{Lemma \ref{lemma:bal-closeness}}
Assume Assumption~\ref{ass:bounded-diagonal}. If $w_{m^*}^{\iter{t}} \ge 1-\varepsilon$, then
\[
\sup_{x\in \mathcal{X}}
\big|\sigma^2_{t,\mathrm{MoG}}(x)-\sigma^2_{t,m^*}(x)\big|
\le
\varepsilon\kappa^2 + \sup_{x\in\mathcal X}\alpha_{t,\mathrm{disag}}(x).
\]
\begin{proof}
Let $m^*\in\{1,\dots,M\}$ satisfy $w_{m^*}^{\iter{t}}\ge 1-\varepsilon$.
By the MoG posterior variance decomposition (Lemma~\ref{lemma:meanvar}),
\[
\sigma^2_{t,\mathrm{MoG}}(x)
=\sum_{m=1}^M w_m^{\iter{t}}\,\sigma^2_{t,m}(x)\;+\;\alpha_{t,\mathrm{disag}}(x),
\]
where $\alpha_{t,\mathrm{disag}}(x)\ge 0$ for all $x$.
Fix any $x\in\mathcal X$. Then
\begin{align*}
\sigma^2_{t,\mathrm{MoG}}(x)-\sigma^2_{t,m^*}(x)
&=\sum_{m=1}^M w_m^{\iter{t}}\,\sigma^2_{t,m}(x)-\sigma^2_{t,m^*}(x)\;+\;\alpha_{t,\mathrm{disag}}(x)\\
&=\sum_{m\neq m^*} w_m^{\iter{t}}\,\sigma^2_{t,m}(x)
-\big(1-w_{m^*}^{\iter{t}}\big)\sigma^2_{t,m^*}(x)
\;+\;\alpha_{t,\mathrm{disag}}(x).
\end{align*}
By Assumption~\ref{ass:bounded-diagonal}, $\sigma^2_{t,m}(x)\le \kappa^2$ for all $m,x$, so
\[
0\le \sum_{m\neq m^*} w_m^{\iter{t}}\,\sigma^2_{t,m}(x)\le (1-w_{m^*}^{\iter{t}})\kappa^2,
\qquad
0\le (1-w_{m^*}^{\iter{t}})\sigma^2_{t,m^*}(x)\le (1-w_{m^*}^{\iter{t}})\kappa^2.
\]
Hence the term
\[
A(x):=\sum_{m\neq m^*} w_m^{\iter{t}}\,\sigma^2_{t,m}(x)
-\big(1-w_{m^*}^{\iter{t}}\big)\sigma^2_{t,m^*}(x)
\]
satisfies $|A(x)|\le (1-w_{m^*}^{\iter{t}})\kappa^2$. Using $\alpha_{t,\mathrm{disag}}(x)\ge 0$,
\[
\big|\sigma^2_{t,\mathrm{MoG}}(x)-\sigma^2_{t,m^*}(x)\big|
=\big|A(x)+\alpha_{t,\mathrm{disag}}(x)\big|
\le |A(x)|+\alpha_{t,\mathrm{disag}}(x)
\le (1-w_{m^*}^{\iter{t}})\kappa^2+\alpha_{t,\mathrm{disag}}(x)
\le \varepsilon\kappa^2+\alpha_{t,\mathrm{disag}}(x).
\]
Taking $\sup_{x\in\mathcal X}$ yields
\[
\sup_{x\in \mathcal{X}}
\big|\sigma^2_{t,\mathrm{MoG}}(x)-\sigma^2_{t,m^*}(x)\big|
\le
\varepsilon\kappa^2 + \sup_{x\in\mathcal X}\alpha_{t,\mathrm{disag}}(x),
\]
as claimed.
\end{proof}

\section{Experimental Details}\label{app:exp_dets}
We use a set of Automatic Relevance Determination (ARD) kernels as the candidate kernel dictionary for our adaptive method. Here are the kernels we used:
\begin{itemize}
    \item \textbf{Squared Exponential (SE):}
    \[
    k(x,x') = \sigma^2 \exp\!\left(
    -\frac{1}{2}\sum_{i=1}^d \frac{(x_i - x_i')^2}{\ell_i^2}
    \right)
    \]
    Infinitely smooth kernel modeling very smooth functions with dimension-specific lengthscales.

    \item \textbf{Matérn 5/2 (M5):}
    \[
    \begin{aligned}
    k(x,x') = {} & \sigma^2\left(
    1 + \sqrt{5}r + \frac{5}{3}r^2
    \right)\exp(-\sqrt{5}r), \\
    & r = \sqrt{\sum_{i=1}^d \frac{(x_i - x_i')^2}{\ell_i^2}}
    \end{aligned}
    \]
    Models moderately smooth functions with twice mean-square differentiability.

    \item \textbf{Matérn 3/2 (M3):}
    \[
    \begin{aligned}
    k(x,x') = {} & \sigma^2\left(1 + \sqrt{3}r\right)\exp(-\sqrt{3}r), \\
    & r = \sqrt{\sum_{i=1}^d \frac{(x_i - x_i')^2}{\ell_i^2}}
    \end{aligned}
    \]
    Models rougher functions with once mean-square differentiability.

    \item \textbf{Rational Quadratic (RQ):}
    \[
    k(x,x') = \sigma^2
    \left(
    1 + \frac{1}{2\alpha}
    \sum_{i=1}^d \frac{(x_i - x_i')^2}{\ell_i^2}
    \right)^{-\alpha}
    \]
    Equivalent to a scale mixture of SE kernels with dimension-wise lengthscales.

    \item \textbf{Periodic (PER):}
    \[
    k(x,x') = \sigma^2 \exp\!\left(
    -2 \sum_{i=1}^d \frac{\sin^2\!\big(\pi|x_i - x_i'|/p_i\big)}{\ell_i^2}
    \right)
    \]
    Models strictly periodic functions with possibly different periods \(p_i\) per dimension.

    \item \textbf{Linear (LIN):}
    \[
    k(x,x') = \sigma^2 \sum_{i=1}^d (x_i - c_i)(x_i' - c_i)
    \]
    Models linear trends with dimension-specific offsets.
\end{itemize}

The training data $(X, y)$ are preprocessed by scaling the inputs $X$ to the unit hypercube $[0,1]$ and standardizing the outputs $y$ to have zero mean and unit variance. At each iteration, the hyperparameters of the GP kernel are re-optimized. For all experiments, we fit a \texttt{SingleTaskGP} model from \texttt{botorch.models}. Most synthetic test functions are taken directly from \texttt{botorch.test\_functions} or \texttt{botorch.test\_functions.synthetic}. For test functions not available in BoTorch, we implement custom classes inheriting from \texttt{SyntheticTestFunction} to ensure consistent behavior and interfaces.

Before being used by the algorithm, the negative log-likelihood (NLL) losses are further transformed to ensure numerical stability and comparability across kernels. Specifically, the losses are first shifted by subtracting the minimum NLL value across candidates, making the smallest loss equal to zero. The shifted losses are then scaled by a constant factor \( C=10 \) and clipped to the interval \([0,1]\). This normalization prevents excessive clamping and also manages to ensure the losses lie in \([0,1]\) for Hedge regret guarantees.

\subsection{Bayesian Optimization}
Log simple regret for each test mentioned below is averaged over 20 seeds and 100 iterations per seed.
\subsubsection{Synthetic Test functions}
\begin{table}[H]
\centering
\caption{Details of the test functions used in the experiments for Bayesian Optimization.}
\label{tab:synthetic_test_functions_bo}
\begin{tabular}{lcc}
\toprule
\textbf{Function} & \textbf{Domain} & \textbf{$d$} \\
\midrule
Branin          & $[-5,10] \times [0,15]$     & $2$ \\
Bukin           & $[-15,-5] \times [-3,3]$   & $2$  \\
Powell          & $[-4,5]^4$                 & $4$  \\
Ackley          & $[-5,5]^d$                 & $5$  \\
Hartmann        & $[0,1]^6$                  & $6$  \\
Rastrigin       & $[-5.12,5.12]^d$           & $8$  \\
Rosenbrock      & $[-5,10]^d$                & $8$  \\
DixonPrice      & $[-10,10]^d$               & $8$  \\
\bottomrule
\end{tabular}
\end{table}

\subsubsection{Robot Pushing Task}
\label{app:robotpushing}
We evaluate on the 3D robot pushing benchmark implemented in a 2D Box2D physics simulator, originally used for Bayesian optimization in robot pre-image learning~\cite{wang2017max} and later adapted as a robust multi-target benchmark~\cite{airbo}. The environment contains a single object (a puck/``ball'') initialized at the origin $o=(0,0)$ and a planar end-effector initialized at $r=(r_x,r_y)$. An action is parameterized by
$
x = (r_x, r_y, t_r)\in [-5,5]\times[-5,5]\times[0,30],
$
where $t_r$ is the push duration.

\paragraph{Action execution}
To keep the action space 3-dimensional, we do not optimize over a push angle. Instead, the end-effector is always commanded to move directly toward the object, i.e., along the deterministic direction
$
u(x) \propto (o-r),
$
equivalently fixing the push angle to face the object (as in prior work). The simulator is stepped for $N=\mathrm{round}(10\,t_r)$ steps and we record the final object position $l(x)\in\mathbb{R}^2$.

\paragraph{Objective}
Following the robust robot pushing setup of~\cite{airbo}, we define four target locations
$
g_1=(-3,-3),\;
g_2=(-3,3),\;
g_3=(4.3,4.3),\;
g_4=(5.1,3.0),
$
and measure the terminal loss as a minimum over targets with mixed squared/linear distances:
\begin{equation}
\ell(x) \;=\; \min\big( \|l(x)-g_1\|_2^2,\; \|l(x)-g_2\|_2,\; \|l(x)-g_3\|_2,\; \|l(x)-g_4\|_2 \big).
\end{equation}
We maximize the reward $f(x)=-\ell(x)$, so that the optimum value is $0$.

\paragraph{Bayesian optimization protocol}
We run Bayesian optimization for $T=100$ iterations using Expected Improvement (EI). Performance is reported as the (log) simple regret, averaged over 20 random seeds. Since $f(x)\le 0$, the final simple regret can be written as
$
r_T = 0 - \max_{t\le T} f(x_t) = \min_{t\le T}\ell(x_t).
$
We use the same baselines and kernel families as in Section~\ref{sec:boexp}.

\subsubsection{Baseline Details}
\begin{enumerate}
    \item \textbf{Uniform Random:}
    A heuristic from~\cite{experimentalAdaptive} that selects a kernel uniformly
    at random from the candidate set.

    \item \textbf{Best Util:}
    A heuristic from~\cite{experimentalAdaptive} that selects, at iteration $t$,
    the kernel with the largest maximum acquisition value.

    \item \textbf{Equal Weights:}
    A uniform ensemble of GPs. It is equivalent to HACK with the MoG predictive
    distribution and a constant zero loss at every round.

    \item \textbf{EGP:}
    The primary method of~\cite{lu2022surrogatemodelingbayesianoptimization}.
    It corresponds to HACK with categorical sampling, Hedge with $\eta=1$ in
    place of AdaHedge, and unbounded NLL as the loss function.

    \item \textbf{CAKE:}
    We use the CAKE implementation provided in the authors' public GitHub
    repository~\cite{suwandi2025adaptivekerneldesignbayesian}.
    We follow their experimental settings
    ($n_c=5$, $p_m=0.7$, and $n_p=10$) and use
    \textsc{gemini-2.0-flash}, which Appendix C of
    \cite{suwandi2025adaptivekerneldesignbayesian} reports to achieve performance
    similar to \textsc{gpt-4o-mini}.
\end{enumerate}

\subsection{Level Set Estimation}
\subsubsection{Synthetic Test functions}

The synthetic test functions used for Level Set Estimation are defined as follows.

\paragraph{Shifted Negative Himmelblau \cite{zanette2018robust}}
Let $x = (x_1, x_2) \in [-5,5]^2$. The function is defined as
\begin{equation}
f_{\text{Him}}(x)
= -\Big[(x_1^2 + x_2 - 11)^2 + (x_1 + x_2^2 - 7)^2\Big] + 100,
\end{equation}

\paragraph{Gaussian Modulated Sinusoid}
Let $x = (x_1, x_2) \in [0,1]^2$. The function is defined as
\begin{equation}
\begin{aligned}
f_{\text{GMS}}(x)
= {} & \sin(2\pi x_1) + \sin(2\pi x_2) \\
& {} + 0.7 \exp\!\left(-\frac{\lVert x - c \rVert^2}{2 s^2}\right)
\sin(20\pi x_1)\cos(20\pi x_2),
\end{aligned}
\end{equation}
where $c = (0.7, 0.35)$ and $s = 0.10$.

\begin{table}[H]
\centering
\caption{Synthetic datasets used for Level Set Estimation}
\label{tab:synthetic_test_functions_lse}
\begin{tabular}{lccc}
\toprule
\textbf{Function} & \textbf{Threshold} & \textbf{$d$} & \textbf{Noise} \\
\midrule
ShiftedNegHimmelblau          & $100.0$     & $2$ & $7.389$ \\
GaussianModulatedSinusoid     & $0.0$       & $2$ & $0.05$ \\
\bottomrule
\end{tabular}
\end{table}

\subsubsection{Visualizations:}
\begin{figure}[H]
\centering
\begin{subfigure}[H]{0.45\linewidth}
    \centering
    \includegraphics[width=0.9\linewidth]{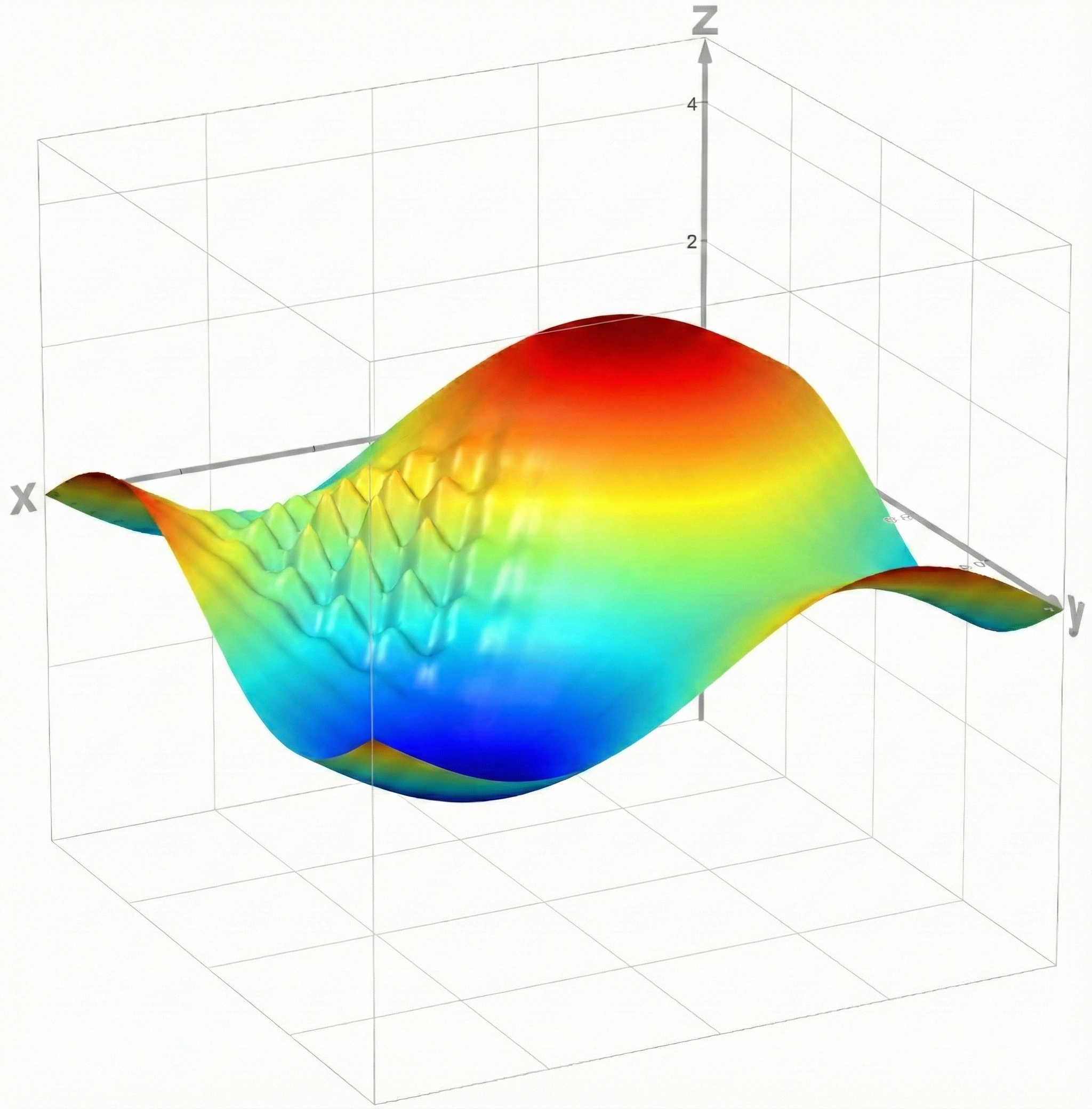}
    \caption{Gaussian Modulated Sinusoid}
\end{subfigure}
\hfill
\begin{subfigure}[H]{0.45\linewidth}
    \centering
    \includegraphics[width=0.9\linewidth]{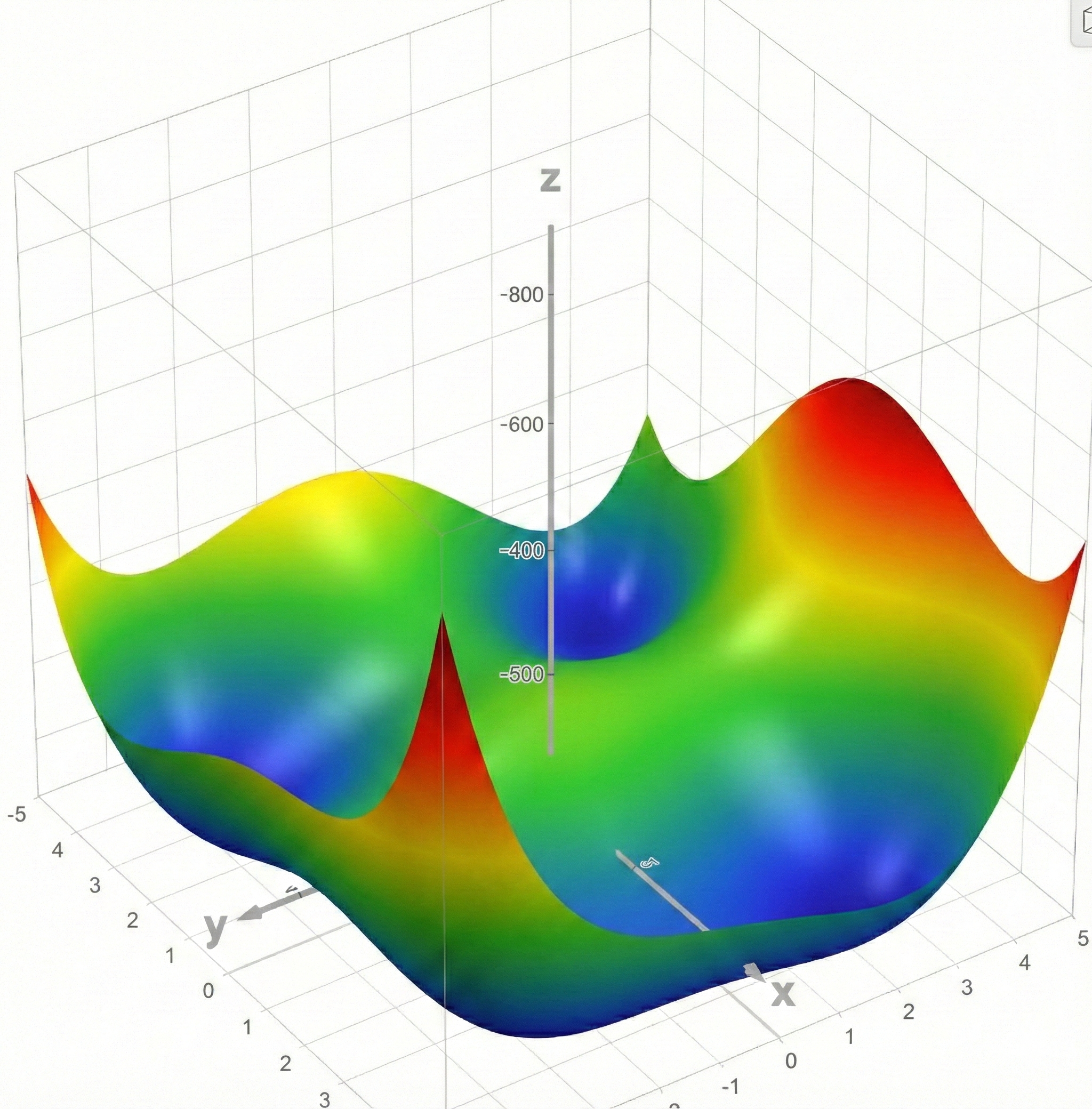}
    \caption{Himmelblau}
\end{subfigure}
\caption{Synthetic benchmark functions used in our experiments.}
\label{fig:synthetic_functions}
\end{figure}

\subsubsection{Real-world dataset}

We use the Motorcycle dataset used in \cite{silverman1985some}. This dataset gives a series of measurements of head acceleration in a simulated motorcycle accident, used to test crash helmets.

\begin{table}[H]
\centering
\caption{Motorcycle dataset}
\label{tab:test_functions_lse}
\begin{tabular}{lcccccc}
\toprule
\textbf{Function} & \textbf{Threshold} & \textbf{$d$} & \textbf{Feature(s)} & \textbf{Label} & \textbf{Number of Instances} & \textbf{Noise} \\
\midrule
Motorcycle          & $0$     & $1$  &  $t$ & $a(t)$  & $133$ & $0.0$ \\
\bottomrule
\end{tabular}
\end{table}

Each data point in this dataset corresponds to the instantaneous acceleration $a(t)$ of the motorcycle helmet measured at a specific time $t$ during the crash event.

\subsection{Bayesian Active Learning for Regression}

\begin{table}[H]
\centering
\caption{Details of the test functions (simulators) used in the experiments for Bayesian Active Learning.}
\label{tab:synthetic_test_functions_bal}
\begin{tabular}{lccc}
\toprule
\textbf{Function} & \textbf{Domain} & \textbf{$d$} & \textbf{Noise} \\
\midrule
Branin          & $[-5,10] \times [0,15]$     & $2$ & $11.32$\\
Gramacy2D        & $[-2, 6]^2$                  & $2$ & $0.05$  \\
Hartmann        & $[0,1]^6$                  & $6$ & $0.0192$ \\
\bottomrule
\end{tabular}
\end{table}

\section{Ablation Study}
\subsection{Choice of Loss Function}\label{app:C1}

\begin{table*}[h]
  \caption{Final log simple regret after $T=100$ BO iterations (Lower is better). Mean values are shown with standard deviation in parentheses. Best performing method is in \textbf{bold}, second best is \underline{underlined}.}
  \label{tab:bo_synthetic_loss}
  \centering
  \small
  
  \rowcolors{4}{white}{gray!6}
  \begin{tabular}{l cccccc cccccc} 
    \toprule
    \rowcolor{white}
    \multirow{2}{*}{Test Function}
      & \multicolumn{6}{c}{CatAcq}
      & \multicolumn{6}{c}{MixAcq} \\
    \cmidrule(lr){2-7} \cmidrule(lr){8-13}
    \rowcolor{white}
      & Zero & NLL & \multicolumn{3}{c}{Combo ($\alpha$)} & Brier 
      & Zero & NLL & \multicolumn{3}{c}{Combo ($\alpha$)} & Brier \\
    \cmidrule(lr){4-6} \cmidrule(lr){10-12}
    \rowcolor{white}
      & & & 0.25 & 0.50 & 0.75 & & & & 0.25 & 0.50 & 0.75 & \\
    \midrule
    
    Branin-2D 
      & \res{-8.81}{1.50} & \res{-8.47}{1.74} & \res{-9.35}{1.40} & \res{-9.02}{1.54} & \textbf{\res{-10.06}{1.76}} & \res{-9.39}{1.83} 
      & \res{-7.10}{2.32} & \res{-9.27}{2.25} & \res{-8.90}{1.74} & \res{-9.48}{1.82} & \res{-8.91}{1.62} & \underline{\res{-9.62}{1.61}} \\ 
      
    Bukin-2D 
      & \res{1.34}{0.48} & \res{1.10}{0.60} & \res{1.05}{0.61} & \res{1.07}{0.81} & \res{1.14}{0.89} & \res{1.37}{0.58} 
      & \res{1.27}{0.57} & \res{1.32}{0.45} & \textbf{\res{0.85}{0.73}} & \underline{\res{0.88}{0.76}} & \res{1.19}{0.46} & \res{1.21}{0.56} \\ 
      
    Powell-4D 
      & \res{2.80}{1.28} & \res{1.69}{0.70} & \res{1.99}{1.16} & \res{2.24}{1.13} & \res{3.14}{1.24} & \res{3.33}{1.08} 
      & \res{2.41}{0.77} & \res{1.44}{1.11} & \textbf{\res{0.90}{1.41}} & \underline{\res{0.95}{1.24}} & \res{1.33}{1.45} & \res{1.67}{0.87} \\ 
      
    Ackley-5D 
      & \res{1.85}{0.24} & \res{2.10}{0.27} & \res{1.86}{0.36} & \res{1.77}{0.34} & \res{1.73}{0.35} & \res{1.80}{0.36} 
      & \underline{\res{1.66}{0.35}} & \res{1.76}{0.23} & \textbf{\res{1.62}{0.23}} & \res{1.69}{0.25} & \res{1.72}{0.31} & \res{1.68}{0.33} \\ 
    
    Hartmann-6D 
      & \res{-4.57}{1.77} & \res{-5.23}{2.00} & \res{-4.92}{2.03} & \res{-4.86}{1.78} & \res{-4.52}{1.94} & \res{-5.07}{2.11} 
      & \res{-5.43}{2.17} & \underline{\res{-5.73}{2.81}} & \res{-5.47}{2.76} & \textbf{\res{-5.99}{2.54}} & \res{-5.39}{2.20} & \res{-5.64}{2.10} \\ 
      
    Rastrigin-8D 
      & \res{3.99}{0.21} & \res{4.20}{0.23} & \res{4.26}{0.20} & \res{4.25}{0.17} & \res{4.17}{0.24} & \res{4.08}{0.25} 
      & \underline{\res{3.85}{0.20}} & \res{4.01}{0.14} & \res{3.93}{0.23} & \res{3.90}{0.24} & \res{3.90}{0.28} & \textbf{\res{3.71}{0.29}} \\ 
    
    Rosenbrock-8D 
      & \res{7.93}{0.57} & \res{7.61}{0.67} & \res{7.38}{0.98} & \res{7.58}{0.80} & \res{7.85}{0.91} & \res{8.05}{0.75} 
      & \res{7.70}{0.66} & \textbf{\res{6.99}{0.69}} & \res{7.57}{0.64} & \underline{\res{7.34}{0.70}} & \res{7.61}{0.63} & \res{7.72}{0.75} \\ 
      
    DixonPrice-8D 
      & \res{5.45}{0.75} & \res{5.15}{0.94} & \res{5.38}{0.86} & \res{5.46}{0.94} & \res{5.61}{1.07} & \res{6.21}{1.36} 
      & \res{5.88}{0.74} & \textbf{\res{4.96}{0.81}} & \res{5.26}{0.66} & \res{5.22}{1.03} & \res{5.56}{0.70} & \underline{\res{5.02}{0.68}} \\ 
    
    \midrule
    \textbf{Total Wins} & 0 & 0 & 0 & 0 & 1 & 0 & 0 & 2 & 3 & 1 & 0 & 1 \\
    \textbf{Avg. Rank} & 9.50 & 7.50 & 7.00 & 7.63 & 8.13 & 9.63 & 7.25 & 4.75 & \underline{3.75} & \textbf{2.75} & 6.00 & 4.13 \\
    \bottomrule
  \end{tabular}
\end{table*}

\noindent\textbf{Table \ref{tab:bo_synthetic_loss}} reports the final log simple regret after \(T=100\) BO iterations across synthetic benchmark functions. We compare Zero, NLL, Brier, and their convex combination
\[
\text{Combo}(\alpha) = \alpha\,\ell_{\text{Brier}} + (1-\alpha)\,\ell_{\text{NLL}},
\]
where \(\alpha \to 0\) recovers NLL and \(\alpha \to 1\) recovers Brier. As \(\alpha\) varies, the combined loss smoothly interpolates between the two losses. Empirically, \textbf{MixAcq with \(\alpha=0.5\)} achieves the lowest average rank, indicating the best overall performance.

\subsection{Hedge vs AdaHedge}\label{app:HedgeVsAdahedge}
\begin{figure*}[h]
    \centering
    \includegraphics[width=1\linewidth]{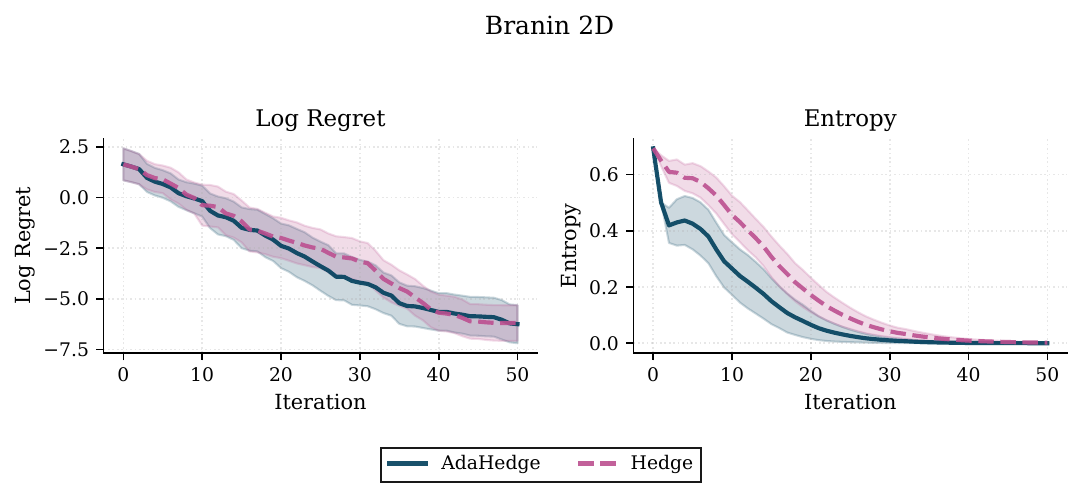}
\caption{
Comparison of AdaHedge and Hedge within HACK-MoG on the Branin-2D benchmark.
Left: log simple regret over BO iterations.
Right: entropy of the kernel-weight distribution, with lower entropy indicating
greater concentration of weight on a smaller number of experts.
Lines show the mean across runs and shaded regions show the corresponding
95\% confidence intervals.
AdaHedge concentrates its weights more rapidly while achieving comparable
or faster regret reduction.
}
\end{figure*}

The figure compares AdaHedge and Hedge as online learning algorithms  on the Branin 2D benchmark using HACK-MoG with the same loss described in Section 5, showing their evolution in terms of log regret and entropy over iterations. We note similar performance in terms of regret with AdaHedge being faster to reach the same regret value. In the entropy plot, AdaHedge’s entropy drops more rapidly, reflecting faster concentration of weights on the better-performing kernels, whereas Hedge maintains higher entropy for longer before gradually converging. This demonstrates the advantage of AdaHedge over Hedge in sample efficient settings where the budget is constrained. 

\subsection{Empirical Evidence for Loss Gap} \label{app:C4}
This ablation study empirically evaluates the loss gap assumptions underlying Lemmas~\ref{lemma:kernelconc} and~\ref{lemma:kernelconcexpectation}, which require a strictly positive loss difference between the best-performing kernel \(m^*\) and competing kernels, either per iteration or in expectation. Figure~\ref{fig:loss_analysis} shows the empirical loss difference \(\ell_m^{(t)} - \ell_{m^*}^{(t)}\) and its running average for the Bukin benchmark, where the Matérn~\(1/2\) kernel emerges as the best kernel, under the combined NLL--Brier loss.

\begin{figure*}[h]
    \centering
    \includegraphics[width=1\linewidth]{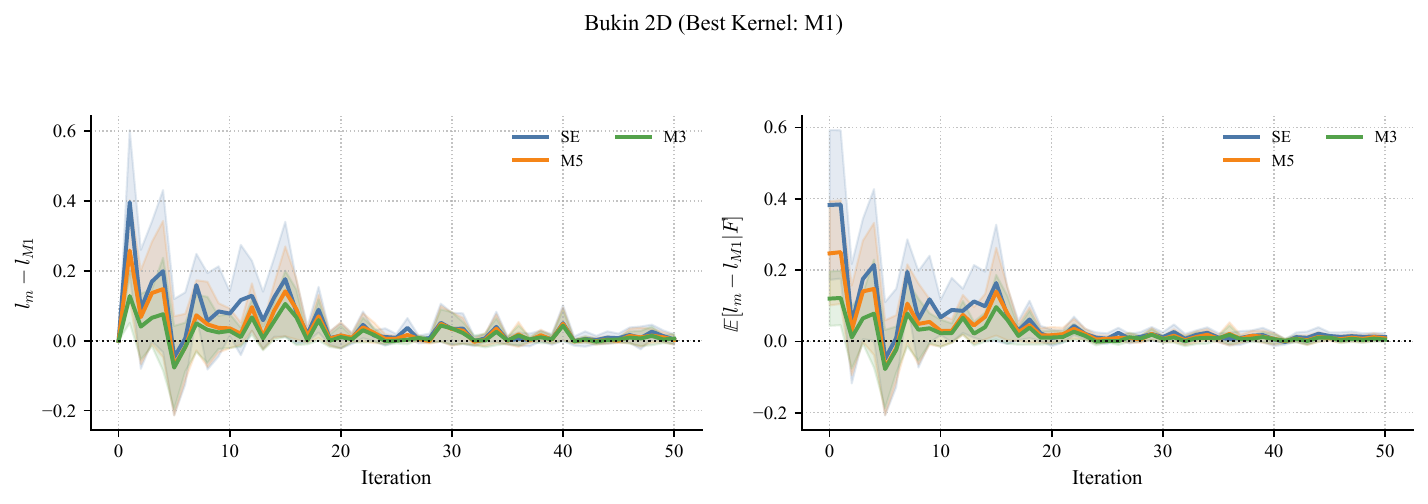}
    \caption{Empirical loss gap for the Bukin 2D benchmark. The left plot reports the per-iteration loss difference, estimated using a Monte Carlo method, as defined in Lemma \ref{lemma:kernelconc} under its stated assumptions, and the right panel reports the expected loss difference, as defined in Lemma \ref{lemma:kernelconcexpectation} under its corresponding assumptions, between the best-performing kernel (Matérn~1/2) and each competing kernel.
Results are shown for the combined NLL and Brier loss function  $\ell^{(t)}_m$ (Section \ref{bo:acq_loss}).}
    
    \label{fig:loss_analysis}
\end{figure*}

\subsection{Runtime Comparison}\label{app:runtime}
\begin{table}[H]
\centering
\caption{Average computational time ($\downarrow$) in seconds per iteration for different methods in Bayesian Optimization.}
\label{tab:runtime_comparison}
\begin{tabular}{lcc}
\toprule
\textbf{Method} & \textbf{Time (s)} & \textbf{Main Bottleneck} \\
\midrule
SE             & 0.15 & Single GP fitting \\
M5             & 0.19 & Single GP fitting \\
Random         & 0.33 & Multiple GP fitting \& weight update \\
Best Util      & 0.34 & Multiple GP fitting \\
Equal Weights  & 0.32 & Multiple GP fitting \\
HACK-Cat     & 0.35 & Multiple GP fitting \& weight update \\
HACK-MoG     & 0.32 & Multiple GP fitting \& weight update \\
\bottomrule
\end{tabular}
\end{table}

Table \ref{tab:runtime_comparison} shows the time taken by the baselines defined in Table \ref{tab:bo_synthetic}, per-iteration time averaged over 200 iterations.

\end{document}